\documentclass[twoside]{article}

\usepackage[accepted]{aistats_mod}

\usepackage[round]{natbib}

\usepackage{mkolar_definitions}
\usepackage{xspace}
\usepackage{enumitem}
\usepackage[dvipsnames,table]{xcolor}
\usepackage{mathabx,graphicx}
\usepackage{thmtools, thm-restate}
\usepackage[mathscr]{euscript}
\usepackage{array, booktabs, boldline} %
\usepackage{cellspace}
\usepackage{wrapfig}
\usepackage{algorithm}
\usepackage[noend]{algorithmic}

\newcommand{\children}{\ensuremath{\mathsf{ch}}}
\newcommand{\cluster}{\ensuremath{\mathcal{C}}}

\newcommand{\hc}{\ensuremath{\mathcal{T}}}
\newcommand{\datapoint}{\ensuremath{x}}
\newcommand{\lbl}{\ensuremath{y}}
\newcommand{\dataset}{\ensuremath{\mathcal{X}}}
\newcommand{\labelset}{\ensuremath{\mathcal{Y}}}

\newcommand{\dataenc}{\ensuremath{f_\theta}}
\newcommand{\lblenc}{\ensuremath{f_\phi}}

\newcommand{\dataencprime}{\ensuremath{f_{\theta'}}}
\newcommand{\lblencprime}{\ensuremath{f_{\phi'}}}

\newcommand{\lblencPTstep}[1]{f_{\phi_{#1}}(\lbl)}
\newcommand{\lblencPTstepNoLbl}[1]{f_{\phi_{#1}}}

\newcommand{\loss}{\ensuremath{\mathcal{L}}}
\newcommand{\ip}[2]{\ensuremath{\langle #1, #2\rangle}}

\newcommand{\dimensionality}{\ensuremath{d}}

\newcommand{\allparameters}{\ensuremath{\Theta}}

\newcommand{\dataparameters}{\ensuremath{\theta}}

\newcommand{\numtraining}{\ensuremath{N}}
\newcommand{\invtemp}{\ensuremath{\beta}}

\newcommand{\lblparameters}{\ensuremath{\phi}}
\newcommand{\E}{\ensuremath{\EE}}

\newcommand{\nystrom}{Nystr\"{o}m\xspace}
\newcommand{\numlandmarks}{\ensuremath{d'}}

\newcommand{\approxsoftmax}{\ensuremath{Q}}

\newcommand{\partition}{\ensuremath{\mathscr{C}}}

\newcommand{\bigo}{\ensuremath{\mathcal{O}}}

\newcommand{\acceptanceratio}{\textsf{A}}

\newcommand{\ours}{\textsc{D}y\textsc{nnibal}\xspace}

\newcommand{\mycomment}[1]{{\color{aurometalsaurus} $\rhd$ #1}}

\definecolor{orange}{HTML}{E07B3E}
\definecolor{purple}{HTML}{775EDB}
\definecolor{green}{HTML}{2D8950}
\definecolor{blue}{HTML}{1858AB}
\definecolor{red}{HTML}{A62216}
\definecolor{teal}{HTML}{009CAD}
\definecolor{ashgrey}{rgb}{0.7, 0.75, 0.71}
\definecolor{aurometalsaurus}{rgb}{0.43, 0.5, 0.5}
\newcount\Comments  
\newcommand{\kibitz}[2]{\ifnum\Comments=1\textcolor{#1}{#2}\fi}

\usepackage{hyperref}
\hypersetup{
  pdffitwindow=true,
  pdfstartview={FitH},
  pdfnewwindow=true,
  colorlinks,
  linktocpage=true,
  linkcolor=Green,
  urlcolor=Green,
  citecolor=Green
}

\begin{document}

\runningtitle{Improving Dual-Encoder Training through Dynamic Indexes for Negative Mining}

\twocolumn[

\aistatstitle{Improving Dual-Encoder Training through \\ Dynamic Indexes for Negative Mining }

\aistatsauthor{Nicholas Monath \And Manzil Zaheer \And  Kelsey Allen \And Andrew McCallum }

\aistatsaddress{ Google Research \And  Google DeepMind \And Google DeepMind \And Google Research }

]

\begin{abstract}
\vspace{-3mm}
Dual encoder models are ubiquitous in modern classification and retrieval. Crucial for training such dual encoders is an accurate estimation of gradients from the partition function of the softmax over the large output space; this requires finding negative targets that contribute most significantly (``hard negatives''). Since dual encoder model parameters change during training, the use of traditional static nearest neighbor indexes can be sub-optimal.  These static indexes (1) periodically require expensive re-building of the index, which in turn requires (2) expensive re-encoding of all targets using updated model parameters. This paper addresses both of these challenges. First, we introduce an algorithm that uses a tree structure to approximate the softmax with provable bounds and that dynamically maintains the tree. Second, we approximate the effect of a gradient update on target encodings with an efficient Nystr\"{o}m low-rank approximation. In our empirical study on datasets with over twenty million targets, our approach cuts error by half in relation to oracle brute-force negative mining. Furthermore, our method surpasses prior state-of-the-art while using 150x less accelerator memory.
\end{abstract}
\vspace{-3mm}
\section{INTRODUCTION}

Dual encoder models map input queries and output targets to a common vector space in which inner products of query and target vectors yield an accurate  similarity function.  
They are a highly effective, widely deployed solution for classification and retrieval tasks 
such as passage retrieval \citep{karpukhin2020dense}, question answering \citep{qu2021rocketqa}, recommendation systems \citep{wu2020joint}, entity linking \citep{gillick2019learning}, and fine-grained  classification \citep{xiong2022extreme}.
These tasks are characterized as having a large number of targets, often on the order of millions to billions.
Dual encoders achieve scalability to this large number of targets in two ways: weight sharing among targets through a parametric encoder and an efficient inner product-based scoring function.
The encoder models are often parameterized as deep neural networks, e.g., transformers \citep{vaswani2017attention,devlin2019bert}, and trained with the cross-entropy loss between the model's softmax distribution and query's true labeled target(s).

Training dual encoder models \emph{efficiently} and \emph{effectively} poses two key challenges:

\noindent\textbf{Computationally intensive loss function.} Computing the gradient of the softmax partition function becomes computationally costly when the number of possible targets is large \citep[inter alia]{bengio2008adaptive,daume2017logarithmic,lindgren2021efficient}, necessitating approximation.  
The common approach approximates the large sum in the partition function gradient by sampling relatively few of its largest terms originating from ``hard'' negative targets, which are often found using an efficient nearest-neighbor index \citep{guu2020retrieval,agarwal2022entity}.
However, the approximation introduces bias in gradient estimation affecting learning and resulting accuracy~\citep{rawat2019sampled, ajalloeian2020convergence}.

\noindent\textbf{Moving embeddings.}
The embedded representations of both queries and targets continuously change during training as the underlying encoder parameters are updated. 
Since re-embedding all targets after each step of training is computationally infeasible,
prior work uses `stale' representations for negative mining, i.e., the vector representation from the encoder parameters from $t$ steps ago.  
Re-encoding and re-indexing even at moderately-sized intervals remains an expensive operation \citep{izacard2022few}. 

In this paper, we present a new dual-encoder training algorithm, which we call \ours,
that addresses both of the above challenges, 
and we provide both theoretical and empirical analysis. To elaborate:

\noindent\textbf{Efficient gradient bias reduction}. 
The expensive term in the exact gradient computation is the evaluation of an expectation with respect to the entire softmax distribution.
To approximate this expectation, we design a Metropolis-Hastings sampler that uses a tree-structured hierarchical clustering of the targets to provide an efficient proposal distribution in \S\ref{sec:sample_cov_sm}. We then relate bias reduction to the granularity of the clustering.
We also present an efficient method for dynamically updating this tree structure in response to encoding parameter gradients. 

\noindent \textbf{Efficient re-embedding}. 
Instead of frequently re-running the updated (typically fairly large and expensive) target encoder model to produce up-to-date embeddings of targets, we propose to approximate (with effective end-task performance) the effect of a gradient update on the target embeddings with an efficient Nystr\"{o}m low-rank approximation (generally using several orders of magnitude less time and memory)  in \S\ref{sec:lowDimApprox}. 

\noindent \textbf{Theoretical analysis}. 
We study the running time of our algorithm under mild assumptions such as Lipschitz encoder functions and expansion rate, showing per-step cost is a function of the number of clusters in our approximation,  considerably smaller than the number of targets itself (Prop.~\ref{proposition:normalizedError}, Prop.~\ref{thm:rejection}). 
We also present a bound on the bias of our softmax estimator, ensuring convergence with stochastic gradient descent (Prop.~\ref{proposition:gradientBias}).

\noindent \textbf{Empirical performance}. 
In \S\ref{sec:exp}, we evaluate our algorithm on two passage retrieval task datasets, {\it Natural Questions (NQ)} \citep{kwiatkowski2019natural} and {\it MSMARCO} \citep{bajaj2016ms}.
On NQ, which has over twenty million targets, we find that our approach cuts error by half in relation to a practically-infeasible oracle brute-force negative mining. Previous state-of-the-art methods can incrementally increase accuracy by increasing memory usage at training time; yet we find that our method still surpasses previous state-of-the-art when using 150x less accelerator memory.
\section{BACKGROUND}
\vspace{-1mm}
Given a data point $\datapoint \in \dataset$ (e.g., a query),
we are tasked with predicting a target $\lbl \in \labelset$ (e.g., a passage answering the question). In our experiments, the number of targets is large, such as tens of millions. 
We assume that the targets $\lbl$ are featurized. We use \emph{encoder} models to represent
both the points and targets. These encoders, which map a point or target's features to a fixed dimensional embedding, are often large parametric pre-trained models (such as transformers \citep{vaswani2017attention,devlin2019bert}).
We denote the encoder model for points as $\dataenc(\datapoint) \in \RR^{\dimensionality}$ and for targets as  $\lblenc(\lbl) \in \RR^{\dimensionality}$. The softmax distribution is:
\begin{equation}
\resizebox{.88\linewidth}{!}{%
$\begin{aligned}
P(\lbl|\datapoint) =\frac{\exp(\invtemp\ip{\dataenc(\datapoint)}{\lblenc(\lbl)})}{Z \triangleq \sum_{\lbl' \in \labelset} \exp(\invtemp\ip{\dataenc(\datapoint)}{\lblenc(\lbl')})},
\end{aligned}$
}\label{eq:softmax}
\end{equation}

where $\invtemp$ denotes the inverse temperature hyperparameter.

We train the parameters of the encoder models $\allparameters = \{\dataparameters,\lblparameters\}$, given labeled training data pairs $(\datapoint_1,\lbl_1),\dots,
(\datapoint_\numtraining,\lbl_\numtraining)$. Our training objective is the cross-entropy loss, which for a given training pair is defined as:
\begin{equation}
    \loss(\datapoint_i,\lbl_i) = - \invtemp\ip{\dataenc(\datapoint_i)}{\lblenc(\lbl_i)} + \log Z. 
\end{equation}
As mentioned by \cite{rawat2019sampled}, most methods use first order optimization, so we consider the gradient wrt $\allparameters$:
\begin{equation}
\begin{aligned}
\nabla_\allparameters \loss(\datapoint_i,\lbl_i) &= - \nabla_\allparameters \invtemp\ip{\dataenc(\datapoint_i)}{\lblenc(\lbl_i)} + \nabla_\allparameters \log Z  \\
\nabla_\allparameters \log Z &= \E_{y \sim P(y|x_i)} \nabla_\allparameters \invtemp\ip{\dataenc(\datapoint_i)}{\lblenc(\lbl)}. \nonumber
\end{aligned}
\end{equation}
Training with cross-entropy loss is computationally challenging because computing the partition function, $Z$, or its gradient requires $|\labelset|$ inner-products. Furthermore, it requires $|\labelset|$ encoding calls to the target encoding model. This latter challenge is unique to training dual encoders and is not a challenge when targets have their own free parameters (e.g., \citet{daume2017logarithmic,sun2019contextual}).

Many works replace the expensive full expectation computation with a Monte Carlo estimate from a small constant number of samples obtained in different ways~\citep{henderson2017efficient,reddi2019stochastic, rawat2019sampled,karpukhin2020dense,lindgren2021efficient}. 
In our work, we design a novel proposal distribution from which it is efficient to sample, and which has bounded gradient bias,
to ensure faster convergence of SGD \citep{ajalloeian2020convergence}. We refer to our approximate loss as $\hat{\loss}$.

Our work will use the same, mild assumptions of many previous works \citep{rawat2019sampled,lindgren2021efficient}.

\begin{restatable}{assum}{lipschitzAssumptions}
\label{assumption:lipchitz} \emph{\textbf{(Lipschitz Encoders)}}
 The dual encoders are $L$-Lipschitz in parameters $\Theta$, that is $\norm{\dataenc(\datapoint) - \dataencprime(\datapoint)} \leq L \norm{\Theta-\Theta'}$ and $\norm{\lblenc(\lbl) - \lblencprime(\lbl)} \leq L\norm{\Theta-\Theta'}$ for all $\datapoint \in \dataset$, $\lbl \in \labelset$.
\end{restatable}
\begin{restatable}{assum}{boundedGradient}
\label{assumption:boundedGradient}
\emph{\textbf{(Bounded Gradients)}} We assume that the logits have bounded gradients, $\norm{\nabla\ip{\dataenc(\datapoint)}{\lblenc(\lbl)}} < M$.
\end{restatable}

\begin{restatable}{assum}{unitNorm}
\label{assumption:unitNorm}
\emph{\textbf{(Unit Norm)}} Dual-encoders produce unit normed vector embeddings\footnote{We note that it is common practice to unit norm the representations from dual encoders, e.g., \citep{gillick2019learning,rawat2020doubly,lindgren2021efficient}}, $\forall \lbl \in \labelset,\ \norm{\lblenc{(\lbl)}} = 1$.
\end{restatable}

\section{EFFICIENT AND ACCURATE SAMPLES FROM THE SOFTMAX DISTRIBUTION}
 \label{sec:sample_cov_sm}
 
In this section, we present our approach for maintaining a dynamic tree-structured clustering that supports efficient and provably accurate sampling from the softmax distribution. In the next section (\S \ref{sec:lowDimApprox}), we will show a novel use of low-rank regression-based approximations to obviate the need for using the encoder to produce updated embeddings and describe the complete training algorithm.  
Proofs for all statements are relegated to the Supplement.
Given the properties of the proposed method, we refer to our approach as \ours, in reference to \textbf{Dy}namic \textbf{N}earest \textbf{N}eighbor \textbf{I}ndex for \textbf{B}ias-reduced \textbf{A}pproximate \textbf{L}oss
(Figure~\ref{fig:main}).

 \newcommand{\rejError}{\epsilon_\textrm{r}}
 \newcommand{\expRejError}{e^{\rejError}}
 \newcommand{\expNegRejError}{e^{-\rejError}}
 \subsection{Accurate Samples from Softmax Distribution}
 \label{sec:cluster_based_proposal}
 We would like to accurately 
 sample from the softmax distribution, without having to compute the computationally intensive
 partition function, $Z$. We consider
 familiar methods: rejection sampling and Metropolis-Hastings. 
 
 To apply rejection sampling, we approximate the unnormalized softmax probability for each target $y$ with an approximation $\hat{y}$ such that:
 \begin{align}
     \expNegRejError \leq \frac{\exp(\invtemp\ip{\dataenc(\datapoint)}{\lblenc(\hat{\lbl})})}{\exp(\invtemp\ip{\dataenc(\datapoint)}{\lblenc(\lbl)})} \leq \expRejError.
 \end{align}
Then, if we sample from: 
\begin{align}
    \lbl \sim \frac{\exp(\invtemp\ip{\dataenc(\datapoint)}{\lblenc(\hat{\lbl})})}{\sum_{\lbl'}\exp(\invtemp\ip{\dataenc(\datapoint)}{\lblenc(\hat{\lbl'})})}
\end{align}
and accept with probability
\begin{align}
  \expNegRejError \frac{\exp(\invtemp\ip{\dataenc(\datapoint)}{\lblenc(\lbl)})}{\exp(\invtemp\ip{\dataenc(\datapoint)}{\lblenc(\hat{\lbl})})},
\end{align}
we will sample from the softmax, akin to past work on rejection sampling for mixture models \citep{zaheer2017canopy}.

Similarly, we can use Metropolis-Hastings to produce a sample from the true softmax distribution $P(\lbl|\datapoint)$ by iteratively sampling (and accepting/rejecting) a state change from a proposal distribution, denoted $\approxsoftmax$.
The approximation error in terms of the total variation of the distribution given by Metropolis-Hastings, $\approxsoftmax_\textsf{MH}$, compared to the true softmax distribution $P$ using a $s$-length chain can be bounded by \citep{mengersen1996rates,cai2000exact,bachem2016fast}:
\begin{equation}
\resizebox{.88\linewidth}{!}{%
$\begin{aligned}
    ||P - \approxsoftmax_\textsf{MH}||_\textsf{TV} \leq \exp \left (-\frac{s-1}{\gamma} \right ) \quad  \gamma = \max_{\lbl \in \labelset} \frac{P(\lbl|\datapoint)}{\approxsoftmax(\lbl|\datapoint)}.
\end{aligned}$}
\end{equation} 
This means that $s>1+\gamma \log \frac{1}{\epsilon}$ gives $||P - \approxsoftmax_\textsf{MH}||_\textsf{TV} \leq \epsilon$.
 
 We achieve high quality samples if the proposal distribution $\approxsoftmax$ is `close' to the true softmax in terms of the ratio $P/\approxsoftmax$. From high quality samples, we will see that bias is minimized to aid convergence of SGD.

 \vspace{-2mm}
 \subsection{Clustering-based Approximations}
 \label{sec:covertree}
 \vspace{-2mm}

The main idea of our approach is to build a clustering of the targets that quantizes the true softmax distribution $P$ by assigning each target to a cluster such that targets in the same cluster have the same probability. Ideally, these clustering-based approximations would be efficient to construct/maintain, and would minimize approximation error.

As an introduction, consider a flat-clustering based approach. We denote the clustering of targets as $\partition$. We denote the cluster assignment of the target $\lbl$ as $\lbl^{(\partition)}$. Each cluster in the clustering $\cluster \in \partition,\ \cluster \subseteq \labelset$ is associated with a representative. We overload notation and use $\cluster$ and $\lbl^{(\partition)}$ to refer to both (1) a set of targets (when used in context of clustering) as well as (2) the features of the cluster's representative (when used as input to a dual encoder model). 
\begin{figure}
     \vspace{-3mm}
     \centering
     \includegraphics[width=0.35\textwidth]{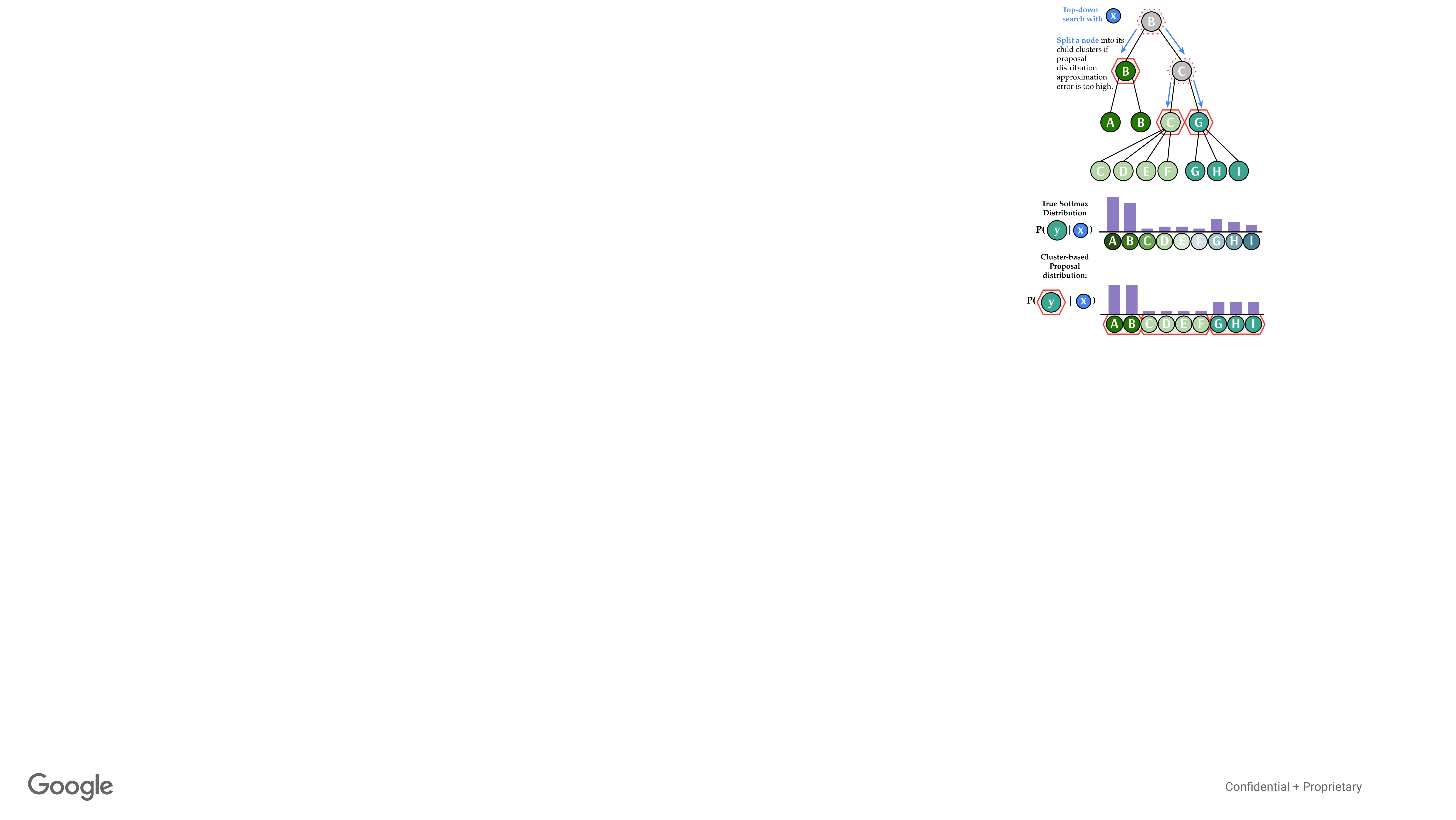}
     \vspace{-2mm}
     \caption{\textbf{Proposed Approach: DyNNIBAL}. Our  approach searches the tree structured index for a clustering of the labels. This clustering is then used to provide an approximate softmax distribution used as a proposal distribution for Metropolis-Hastings. The tree structure supports efficient updates as parameters of the dual encoders change.}
     \label{fig:main}
     \vspace{-3mm}
 \end{figure} 
 
We define the distribution $\approxsoftmax$, by replacing each target with a representative for its cluster. 
\begin{equation}
\resizebox{.88\linewidth}{!}{%
$\begin{aligned}
\approxsoftmax(\lbl|\datapoint;\partition) =\frac{\exp(\invtemp\ip{\dataenc(\datapoint)}{\lblenc(\lbl^{(\partition)})})}{\hat{Z}\triangleq\sum_{\cluster \in \partition} | \cluster| \exp(\invtemp\ip{\dataenc(\datapoint)}{\lblenc(\cluster)})},
\end{aligned}$
}\label{eq:approxclustersoftmax}
\end{equation}

The complexity of sampling from this distribution is a function of the number of clusters, or rather, we only need to measure similarity between the datapoint $\datapoint$ and each of the cluster representatives, e.g., $\bigo(|\partition|)$ inner products to compute the unnormalized probabilities and partition function. 
Next, we investigate how to discover a clustering of targets, which provides efficient sampling while bounding the error of Metropolis-Hastings and increasing the probability of acceptance for rejection sampling.

\vspace{-2mm}
\newcommand{\ctbase}{b}

\subsection{Hierarchical Clustering Structures}
\vspace{-2mm}
\label{sec:search}

Our method for discovering such a clustering is to use hierarchical clustering, in particular SG Trees \citep{zaheer2019sg}, an efficient variant of cover trees \citep{beygelzimer2006cover}. These efficient hierarchical approaches will allow us to bound the approximation quality and discover an adaptive clustering of targets for each point.

\begin{definition}\textbf{\emph{(Hierarchical Clustering)}}
A hierarchical clustering $\hc$ is a set of nested clusterings. For  $\cluster_\mathsf{par}$, `child' clusters are $\children(\cluster_\mathsf{par})$ s.t. $\forall \cluster_\mathsf{kid} \in \children(\cluster_\mathsf{par}), \cluster_\mathsf{kid} \subsetneq \cluster_\mathsf{par}$ and $\nexists \cluster' \in \hc$ s.t. $\cluster_\mathsf{kid} \subset \cluster' \subset \cluster_\mathsf{par}$. 
\end{definition}
\vspace{-1mm}
Cover trees, originally proposed as a nearest neighbor index, are highly scalable to hundreds of millions of targets \citep{zaheer2017canopy} and enjoy theoretical guarantees. 

\begin{definition}\emph{(\textbf{Cover Tree} \citep{beygelzimer2006cover})} A cover tree with base $\ctbase$ is a \emph{level-wise} structure. Levels are clusterings, $\partition_{(\ell)}$. The set of cluster representatives in a level is $Y_{(\ell)}$.
A cover tree maintains the invariants: 
\begin{enumerate}[topsep=0pt,itemsep=0ex,partopsep=1ex,parsep=1ex,leftmargin=5mm]
    \item \textbf{Nesting}. The cluster representatives at a parent level are a subset of the child level, i.e., $Y_{(\ell)} \subseteq Y_{(\ell-1)}$.
    \item \textbf{Covering}. For all $y \in Y_{(\ell-1)}$, there exists a parent node $y_\mathsf{par} \in Y_{(\ell)}$ such that $\norm{\lblenc(y) - \lblenc(y_\mathsf{par})} \leq \ctbase^\ell $
    \item \textbf{Separation}. All distinct nodes in a given $y,y' \in Y_{(\ell)}$, satisfy $\norm{\lblenc(y) - \lblenc(y')} \geq \ctbase^\ell $.
\end{enumerate}
\end{definition}
Closely related are SG Trees  ~\citep{zaheer2019sg}:

\begin{definition}\emph{(\textbf{Stable Greedy (SG) Tree} \citep{zaheer2019sg})} An SG tree is a cover tree with separation defined to only apply to siblings rather than nodes in the same level: 
\begin{enumerate}
[topsep=0pt,itemsep=0ex,partopsep=1ex,parsep=1ex,leftmargin=5mm]
  \setcounter{enumi}{2}
    \item \textbf{Separation}. All distinct siblings, $y,y' \in \children{(y_\mathsf{par})}$, satisfy $\norm{\lblenc(y) - \lblenc(y')} \geq \ctbase^\ell $.
\end{enumerate}
\end{definition}
Because SG trees are considerably more efficient to construct \citep{zaheer2019sg}, they are the focus of our work. However, where possible, we will also describe how cover trees could be used in our methods.

The cover tree and SG tree data structures are not new to this paper. Their use to provide an approximation to the softmax is novel and is the contribution of this section. 

 We make standard assumptions about the representations \citep[inter alia]{beygelzimer2006cover,zaheer2019sg}. 

\begin{restatable}{definition}{expansionConstant}
The expansion constant $\alpha$ for the encoded  targets $Y = \{\lblenc{(\lbl)}: \lbl \in \labelset\}$ is the smallest $\alpha \geq 2$ such that $|B(p, 2r)| \leq \alpha |B(p,r)|$ for all $p \in Y$ and $r \geq 0$ where $B(p,r)$ denotes a ball of radius $r$ around $p$.
\label{def:expansionConstant}
\end{restatable}

First, notice how a particular level of the tree structure can serve as a clustering used in the approximate distribution. Each cluster in the given level has bounded radius over its descendants (cluster members).
By bounding the approximation error of the unnormalized and normalized probabilities for a selected clustering, we can determine the quality of samples using Metropolis-Hastings as well as  the acceptance probability for rejection sampling.

\begin{restatable}{proposition}{unnormalizedError}
\label{proposition:unnormalizedError}
    Under  Assumption~\ref{assumption:unitNorm}, given the clustering at level $\ell$, $\partition_{(\ell)}$, approximating $\lbl$ with the cluster representative, $\lbl^{(\partition_{(\ell)})}$, satisfies the following with $\rejError = \beta \cdot \ctbase^\ell$:
    \vspace{-1mm}
     \begin{align}
     \expNegRejError \leq \frac{\exp(\invtemp\ip{\dataenc(\datapoint)}{\lblenc(\lbl^{(\partition_{(\ell)})})})}{\exp(\invtemp\ip{\dataenc(\datapoint)}{\lblenc(\lbl)})} \leq \expRejError.
 \end{align}
\end{restatable}

Similarly, to achieve a given bound on the total variation for Metropolis-Hastings, we would like to control $\gamma$, the ratio of the true softmax to our proposal distribution, $\max_{\lbl \in \labelset} \frac{P(y|x)}{Q(y|x;\partition_{(\ell)})} = \gamma$, in terms of the level $\ell$ selected.

\begin{restatable}{proposition}{normalizedError}
\label{proposition:normalizedError}
Given Assumption~\ref{assumption:unitNorm}, to achieve a maximum ratio of true softmax to proposal distribution equal to $\gamma$ i.e., $\max_{\lbl \in \labelset} \frac{P(y|x)}{Q(y|x;\partition_{(\ell)})} = \gamma$, we need the clustering at level $\ell$, where:
$
    \ell \triangleq \max \{\ell \in \mathbb{Z} \ : \ \ctbase^{\ell}  \leq \frac{1}{2\beta} \log \gamma  \}.  
$
\end{restatable}

\begin{restatable}{remark}{changeFlat}
Observe the relationship between $\gamma$ and $\ell$. Descending one more level of the tree to level $\ell-1$, reduces the ratio from $\gamma$ (selecting level $\ell$) to $\gamma^{\frac{1}{b}}$.
\end{restatable}

\vspace{-1mm}
The aforementioned results describe how to achieve a given approximation error by selecting a level of the tree structure. 
Now, let's consider methods which adaptively
use the hierarchical structure to produce 
a sample for a given datapoint. These approaches 
will start with a coarse clustering (e.g., some level $\ell$ satisfying a minimal requirement on the aforementioned sources of error).

First, let's consider a theoretically motivated rejection sampling approach based on the rejection sampling methods for mixture models proposed by \cite{zaheer2017canopy}. The approach works by
iteratively selecting a finer-grained set of cluster representatives to sample from while still being a valid rejection sampler for the softmax distribution. The sampling is modified such that if a given cluster is accepted, we return its representative. We start with the clusters at level $\ell$. We perform one step of rejection sampling. If we accept, we return the cluster representative itself. Otherwise, we descend to that 
cluster's children in level $\ell-1$ and repeat the following procedure until the leaves of the tree. We sample among the children of the node and one specially defined \emph{restart} option, $\uptodownarrow$. If the \emph{restart} option is sampled, we begin the algorithm again at level $\ell$. If we sample a child other the nested self-child and we accept the child's cluster, we return its representative with a given probability. If we sample the nested-child or if we do not accept the sampled child, we descend the tree and consider the children of the sampled node. To be a valid rejection sampler, we maintain running normalizers as shown in Algorithm~\ref{alg:rejection_sample_alg}.

\begin{restatable}{proposition}{rejectionSamplingThm}
Algorithm~\ref{alg:rejection_sample_alg} produces samples from the softmax $P(\lbl|\datapoint)$ in time $\bigo(|\partition_{\ell}| + \alpha^4e^{\beta \ctbase^{\ell+2}})$ for cover trees and $\bigo(|\partition_{\ell}| + \alpha^3 e^{\beta \ctbase^{\ell+2}})$ for SG trees. 
\label{thm:rejection}
\end{restatable}

\vspace{-2mm}
\begin{wrapfigure}{r}{0.51\linewidth}
\vspace{-7mm}
\hspace*{-7mm}
    \includegraphics[width=0.3\textwidth]{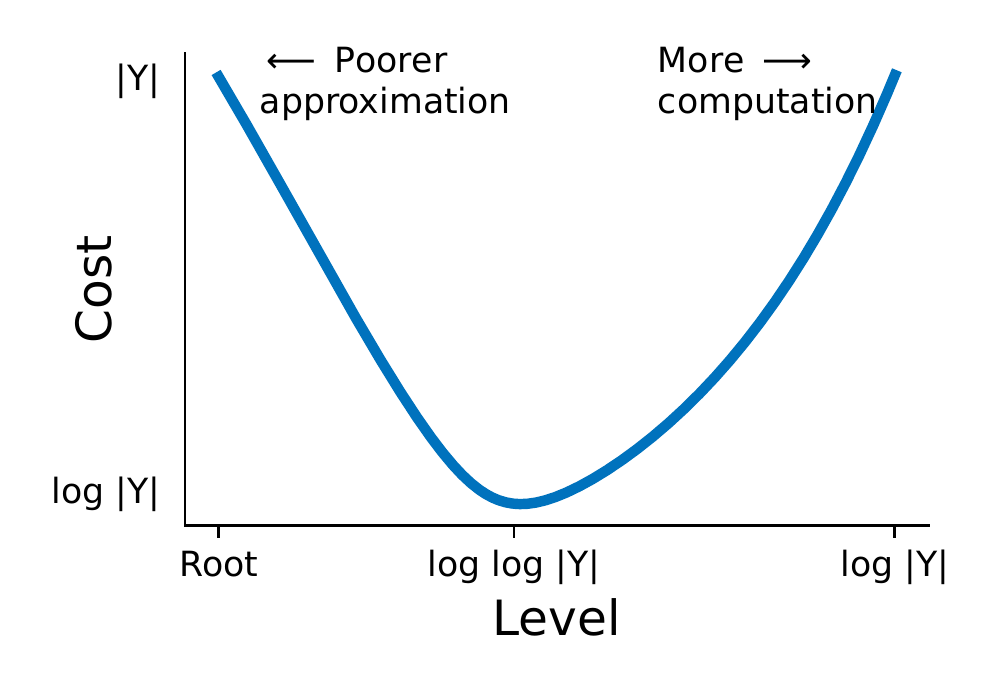}
    \vspace{-9mm}
\end{wrapfigure}
In Proposition~\ref{thm:rejection}, the first term of the cost corresponds to initial computation and second term to quality of the proposal.
If we pick $\ell$ to be too close to the root, then initial computation is cheaper as cluster size $|\partition_{\ell}|$ will be small, but quality of the proposal will be worse leading to many rejections. As we descend down the down the tree for picking $\ell$ the quality of proposal will keep improving but so will initial computation cost as $|\partition_{\ell}|$ grows. There is a good trade-off point in between as illustrated in the figure.

\setlength{\textfloatsep}{5pt}
\newcommand{\numdesc}{\blacktriangle}
\begin{algorithm}[t]
\caption{\textsc{RejectionSampling}}
\begin{algorithmic}[1]
\STATE{\textbf{Input:} $\hc$: tree, $\datapoint$: point, $\ell$: initial level.} 
\STATE{\textbf{Output:} A sample from $P(\lbl|\datapoint)$}
\STATE Define $\numdesc_\cluster \gets \frac{|\cluster|}{ |\labelset|}$ for all tree nodes $\cluster \in \hc$.
\STATE $Z_\ell \gets e^{\ctbase^\ell} \sum_{\cluster \in \partition_i} \numdesc_\cluster \exp(\invtemp\ip{\dataenc(\datapoint)}{\lblenc(\cluster)})$
\STATE Define $\delta_{\ell,\cluster} \gets {Z_\ell}^{-1} e^{\ctbase^\ell}  \numdesc_\cluster \exp(\invtemp\ip{\dataenc(\datapoint)}{\lblenc(\cluster)})  $
\STATE Sample $\cluster_\ell$ from $\partition_\ell$ proportional to  $\delta_{\ell,\cluster_\ell}$
\STATE Accept \& \textbf{return} the representative of $\cluster_\ell$ with prob:\vspace{-2mm} \[\pi_\ell \gets {Z_\ell}^{-1} {\delta_{i,\cluster_i}}^{-1} {|\labelset|}^{-1} \exp(\invtemp\ip{\dataenc(\datapoint)}{\lblenc(\cluster_i)})\] \vspace{-6mm}
\FOR{$k$ \textbf{from} $\ell-1$ \textbf{down to} $-\infty$}
\STATE $Z_k \gets \delta_{k+1}(1-\pi_{k+1})$
\STATE $\delta_{k,\cluster_k} \gets Z_k^{-1} e^{\ctbase^k} \numdesc_{\cluster_k} \exp(\invtemp\ip{\dataenc(\datapoint)}{\lblenc(\cluster_k)})$
\STATE $\delta_{k,\uptodownarrow} \gets 1-\sum_{\cluster \in \children{(\cluster_{k+1}})} \delta_{k,\cluster}$ 
\STATE Sample $\cluster_k$ from $\children{(\cluster_{k+1})}  \bigcup \{\uptodownarrow\}$ prop. to $\delta_{k,\cluster_k}$.
\IF{$\cluster_k =\ \uptodownarrow$}
\STATE \textsc{RejectionSampling}($\hc,\datapoint,\ell$)
\ELSIF{the representative of $\cluster_k$ \& $\cluster_{k+1}$ differ}
\STATE Accept \& \textbf{return} the rep. of $\cluster_k$ with prob $\pi_{k}$:
\vspace{-2mm}
\[
\pi_{k} \gets Z_k^{-1}{\delta_{k,\cluster_k}}^{-1}{|\labelset|}^{-1}\exp(\invtemp\ip{\dataenc(\datapoint)}{\lblenc(\cluster_k)})
\]
\vspace{-4mm}
\ENDIF
\ENDFOR
\end{algorithmic}
\label{alg:rejection_sample_alg}
\end{algorithm} 

Next, we consider a more practical and simple adaptive extension for Metropolis-Hastings. Intuitively, we want to reduce approximation error for high probability targets. Our approach splits a cluster if the radius is at least $\ctbase^m$ and if it \emph{may} contain a target $\lbl$, such that  $\norm{{\dataenc(\datapoint)}-{\lblenc(\lbl)}} < \ctbase^m$. This follows the intuition that we would like to emphasize the closest targets to a point.
In other words, we descend the tree from level $\ell$ to level $m$ splitting clusters which might contain a target that is within $\ctbase^{m}$ of $\dataenc{(\datapoint)}$ (Algorithm~\ref{alg:covertreesearchforpartition}).

\begin{restatable}{proposition}{searchCorrectness}
\label{prop:searchCorrectness} 
Let $\partition$ be the output of Algorithm~\ref{alg:covertreesearchforpartition}, then $\max_{\lbl \in \labelset} \frac{P(\lbl|\datapoint)}{Q(\lbl|\datapoint,\partition)} \leq \gamma$ under Assumption~\ref{assumption:unitNorm}.
\end{restatable}
\begin{restatable}{remark}{searchTime}
While in the worst case, we select all clusters in level $m$, e.g., $\bigo(|\partition_{(m)}|)$ clusters, in practice we expect this to be much less and can limit to a specified number by either limiting the size of the frontier or output partition.
\end{restatable}
\vspace{-2mm}

To put our results in perspective, 
consider a simple alternative, uniform negative sampling. 
Let $Q_\textsf{unif}$ be a uniform proposal distribution, for which $\max_{y \in \labelset} \frac{P(y|x)}{Q_\textsf{unif}(y|x)} \leq |\labelset|$. 

\begin{remark}\textbf{\emph{(Uniform Negatives vs DyNNIBAL)}}
Consider the case where we have fixed compute budget for the chain length. 
For a uniform distribution, the total variation is bounded by $\exp\left(-\bigo\left(\frac{s}{|\labelset|}\right)\right)$.
Since our each sample is slightly more expensive, we can only afford to have a chain of length $\frac{s}{|\partition_{(\ell)}|}$ for any selected clustering $\partition_{(\ell)}$.
But even this reduced length chain will yield a much better total variation bound by Proposition~\ref{proposition:unnormalizedError}.
In particular, if we pick a level just $\log\log|\labelset|$ below the root, which is not very deep, then we obtain the total variation bound as $\exp\left(-\bigo\left(\frac{s}{\log|\labelset|}\right)\right)$.
Notice that this is marked improvement because of the logarithmic term in the denominator. 
\end{remark}

\begin{algorithm}[t]
\caption{\sc FindClustering}
\begin{algorithmic}[1]
\STATE{\textbf{Input:} $\hc$: tree, $x$: point, $\gamma, m$: allowed  error} 
\STATE{\textbf{Output:} A clustering $\partition$ }
\STATE $\ell \gets \max \{\ell \in \mathbb{Z} \ : \ \ctbase^{\ell}  \leq \frac{1}{2\beta} \log \gamma   \}$
\STATE $\mathscr{F}_{\ell} \gets \partition_{\ell}$ \mycomment{Initialize frontier to be the $\ell^{th}$ level of $\hc$.}
\STATE $\partition \gets \{\}$  \mycomment{The output clustering}
\FOR{$k$ \textbf{from} $\ell$ \textbf{down to} $m$}
\STATE $\mathscr{F} \gets \{\children{(F)}: F \in \mathscr{F}_k\}$
\STATE $\mathscr{F}_{k-1} \gets \{\}$
\FOR{$F$ in $\mathscr{F}$}
\IF{$\norm{\dataenc(\datapoint) - \lblenc(F)} > \ctbase^k + \ctbase^{m}$}
\STATE $\partition \gets \partition \cup \{F\}$
\ELSE 
\STATE $\mathscr{F}_{k-1} \gets \{F\}$
\ENDIF
\ENDFOR
\ENDFOR
\STATE \textbf{return} $\partition \cup \mathscr{F}_{m}$
\end{algorithmic}
\label{alg:covertreesearchforpartition}
\end{algorithm} 

\subsection{Gradient Bias of Our Estimator}
\label{sec:covertree_estimaor}
\vspace{-2mm}
To ensure convergence in gradient descent, we need our estimator to have bounded gradient bias \citep{ajalloeian2020convergence}. We are interested in the bias of the gradient estimate: $\norm{\E [\nabla_\Theta \hat{\loss}] - \nabla_\Theta \loss}$, where the expectation is over the Metropolis-Hasting samples. 

\begin{restatable}{proposition}{gradientBias}
\label{proposition:gradientBias} Let $P$ be the true softmax and $Q_\mathsf{MH}$ be the  Metropolis-Hastings approximation to the softmax. Under Assumption~\ref{assumption:boundedGradient}, we have 
$\norm{\E [\nabla_\Theta \hat{\loss}] -  \nabla_\Theta \loss} \leq 2\epsilon \beta M$, where $\norm{P-Q_\mathsf{MH}}_\mathsf{TV} \leq \epsilon$.
\end{restatable}
\vspace{-2mm}
\subsection{Dynamic Maintenance of the Tree Structure}
\label{sec:dynamic}
\vspace{-2mm}

During training, the parameters of the dual encoder models, $\dataenc$, $\lblenc$ are updated. As a result the tree structure properties may no longer be upheld. In this section, we analyze how representations could change under standard assumptions about the data. We then describe an algorithm for maintaining an SG tree and a simple approximation in practice. 

Finding the part of the SG tree that no longer maintains its invariants after a parameter update depends on how the distance between a pair of targets changes after $w$ steps of gradient descent. Let the learning rate be $\eta$. The dual encoder parameters are updated at step $t$ as $\allparameters_t \gets \allparameters_{t-1}  - \eta \nabla_\allparameters \hat{\loss}(\datapoint_t,\lbl_t)$. We can bound the pairwise change:

\begin{restatable}{proposition}{pairwiseAmountOfChange}
\label{proposition:pairwiseAmountOfChange} Under Assumptions~\ref{assumption:lipchitz},\ref{assumption:boundedGradient},\ref{assumption:unitNorm},
let $\phi_t$ and $\phi_{t+w}$ refer to encoder parameters after $w$ more steps of gradient descent with learning rate $\eta$.
\vspace{-1mm}
\begin{equation}
\resizebox{.9\linewidth}{!}{%
$\begin{aligned}
   \left | \norm{\lblencPTstep{t} - f_{\phi_{t}}(\lbl')}_2 - \norm{\lblencPTstep{t+w} - f_{\phi_{t+w}}(\lbl')}_2 \right | \leq 4w\eta \beta LM.
    \label{eq:changedAmount}
\end{aligned}$}
\end{equation}
\end{restatable}
\vspace{-1mm}
We can detect whether, for a given pair of tree nodes with representatives $y$ and $y'$, if the covering property with respect to level $\ell$ will be maintained after the gradient update:
\begin{align}
\norm{\lblencPTstep{t} - f_{\phi_{t}}(\lbl')}_2 +  4w\eta \beta LM \leq \ctbase^{\ell},
\end{align}
and similarly for separation:
\begin{align}
\norm{\lblencPTstep{t} - f_{\phi_{t}}(\lbl')}_2 -  4w\eta \beta LM \geq \ctbase^{\ell-1}.
\end{align}

This leads to a simple algorithm for detecting which parts of the tree structure need to be re-arranged. We can maintain for each node, the distance to its farthest descendant, denoted $\textsc{maxd}(\cluster)$, to detect if covering is maintained. Similarly, we can maintain for each node, the distance between its closest pair of children nodes (SG tree), denoted $\textsc{mind}(\cluster)$, to detect if separation is maintained. Notice however that even if an ancestor node maintains the properties, a descendant may still violate them.  

An algorithm for updating the SG tree is: delete and rebuild the smallest subtrees such that levels above the subtree root maintain covering and separation (checking $\textsc{maxd}$ and $\textsc{mind}$)). We notice that for SG Trees since the separation property is only maintained between siblings and not all nodes in a given level, when we rebuild the structure we can re-attach the rebuilt subtree with the same root. This is described in Algorithm~\ref{alg:update_sg}. 
We observe that the selected bound $4w\eta \beta LM$ is equivalent to picking a level at which we rebuild, which can be easier in practice.
Rebuilding subtrees can be empirically very efficient. The rebuilding can be done independently in parallel. Each subtree contains relatively far fewer targets than the tree as a whole.

\begin{algorithm}[t]
\caption{\sc UpdateSGTree}
\begin{algorithmic}[1]
\STATE{\textbf{Input:} $\cluster_\ell$: a subtree root at level $\ell$, $\cluster_r$: an ancestor node covering  its descendants or $\emptyset$ to indicate we are rebuilding at the top level, $4w\eta \beta LM$: bound on change}
\IF {$\ctbase^\ell \leq 4w\eta \beta LM$}
\STATE Rebuild the subtree $\cluster_\ell$ under $\cluster_r$.
\ELSE
\STATE $\textsc{maxd}(\cluster_\ell)$ is the max dist btw $\cluster_\ell$ to a descendant. \\
\STATE $\textsc{mind}(\cluster_\ell)$ is the min dist btw children of $\cluster_\ell$. \\
\STATE{$\textsf{cov} \gets \textsc{maxd}(\cluster_\ell) +  4w\eta \beta LM \leq \ctbase^\ell$}
\STATE{$\textsf{sep} \gets \textsc{mind}(\cluster_\ell) -  4w\eta \beta LM \geq \ctbase^{\ell-1}$}
\IF{$\textsf{cov}$ and $\textsf{sep}$}
\FOR{$\cluster_{\ell-1} \in \children{(\cluster_{\ell})}$}
\STATE \textsc{UpdateSGTree}($\cluster_{\ell-1}, \cluster_{\ell}, 4w\eta \beta LM$)
\ENDFOR
\ELSE 
\STATE Rebuild the subtree $\cluster_\ell$ under $\cluster_r$.
\ENDIF
\ENDIF
\end{algorithmic}
\label{alg:update_sg}
\end{algorithm}

\vspace{-3mm}
\section{EFFICIENT RE-ENCODING}
\label{sec:lowDimApprox}
\vspace{-2mm}

We address further computational bottlenecks. First, the running time (and space) is a function of the dimensionality of dual encoder embeddings. This dimensionality can typically be quite large ($d=768$  \citep{devlin2019bert}). 
Second, we previously implicitly described the representations of the targets stored in the tree, as $\lblenc{(\cluster)}$ for some cluster representative $\cluster$. We do not re-apply the encoder every time we compare a given data point's embedding to a cluster representative. Instead, we use a cached version of the target embedding. However, re-running 
the dual encoder to re-encode and update this cache (say every $w$ steps of gradient descent) would be very time-consuming and require use of hardware accelerators (GPU/TPU).

We present approaches for addressing each of these computational burdens. The \nystrom method is used to reduce dimensionality. Then we use these low-dimensional representations in a low-rank regression model that approximates the re-running of the encoder model to produce the newest encoded representations (no accelerator required).  

\vspace{-3mm}
\subsection{Reducing Dimensionality}
\vspace{-2mm}

The \nystrom method factorizes a pairwise similarity matrix by representing each row and column (e.g., targets or datapoints) as a $\numlandmarks$ dimensional vector \cite[inter alia]{williams2000using,kumar2012sampling,gittens2013revisiting}. Each dimension of the $d'$ dimensional vector can be thought of as the (scaled) pairwise similarity between the representations of the row/column and a \emph{landmark} representative associated with the particular dimension. Let $\mathbf{K} \in \mathbb{R}^{n \times n}$ be the pairwise similarity matrix, $\mathbf{S} \in \{0,1\}^{n \times \numlandmarks}, \forall j \in [\numlandmarks] \sum_{i\in n} \mathbf{S}_{ij} = 1$ is an indicator matrix corresponding to the sampled landmarks, $n = |\labelset| + |\dataset|$. The \nystrom approximation is then,
    $\mathbf{K}\mathbf{S} \left (\mathbf{S}^T\mathbf{K}\mathbf{S}\right)^{-1}\mathbf{S}^T\mathbf{K}$.
It is important to note that we do not actually \emph{explicitly} instantiate $\mathbf{K}$. We sample our landmarks, compute the pairwise similarity between points/targets and the landmarks to compute $\mathbf{K}\mathbf{S}$ and the landmarks themselves $(\mathbf{S}^T\mathbf{K}\mathbf{S})^{-1}$.
The low-dimensional representation of a target is the corresponding row in  $\mathbf{K}\mathbf{S} \left (\mathbf{S}^T\mathbf{K}\mathbf{S}\right)^{-1}$ or $\mathbf{S}^T\mathbf{K}$. We can approximate all of the pairwise distance computations needed with an inner product of two $\numlandmarks$ dimensional vectors, where $\numlandmarks \ll d$.

\vspace{-3mm}
\subsection{Regression-based Approximation of Re-Encoding}
\label{sec:approxReencode}
\vspace{-2mm}
After $w$ steps of gradient descent,  our model parameters $\allparameters_{t}$ have been updated to be $\allparameters_{t+w}$. We would like to approximate the re-encoding of targets $\lblencPTstep{t+w}$ without having to run the encoder model. Let $Y_t$ be the set of cached target embeddings after $t$ steps, where $\vec{y}$ represents the cached vector for the target $y$. We want to build $Y_{t+w}$ and we will do so by training a regression model, $\mathscr{R}$ to map each $\vec{y} \in Y_t$ to $\lblencPTstepNoLbl{t+w}(y)$. We build $s'$ training data pairs of the form, $(\vec{y_1},\lblencPTstepNoLbl{t+w}(y_1)),\dots(\vec{y_{s'}},\lblencPTstepNoLbl{t+w}(y_{s'}))$. We then fit $\mathscr{R}$ using kernel ridge regression for which we use the above \nystrom approximation to the kernel matrix. 

To update the cached target representations, we apply the regression model $\mathscr{R}$ to every cached target in $Y_{t}$, creating the new cache of target representations $Y_{t+w}$. This provides an extreme speedup in re-encoding time. Instead of $\bigo(|\labelset|)$ calls to the encoder, we have $\bigo(s')$ encoder calls where $s'$ is the number of training examples for $\mathscr{R}$ and $s' \ll |\labelset|$ along with one application of $\mathscr{R}$ to each target, effectively a multiply of a $d'\times d'$ matrix.

Further details can be found in Appendix~\ref{appendix:LowDim}. The methodological techniques of \nystrom and low-rank regression are not new, of course; our contribution is the use of these techniques to facilitate efficient re-encoding, which is a costly and time consuming step in dual encoder training.

\vspace{-3mm}
\subsection{Complete Training Algorithm}
\vspace{-2mm}

In Algorithm~\ref{alg:dynnibalComplete}, we summarize \ours. Initially, we are given a set of targets $\labelset$. We select (randomly) landmarks to apply \nystrom. We encode and reduce the dimensionality of the targets. We construct an SG tree $\hc$ over these targets. 
For a given training example $\datapoint_t$ with label $\lbl_t$, we use Algorithm~\ref{alg:covertreesearchforpartition} (with inputs of $\hc$, the datapoint $\datapoint_t$ and a given $\gamma$ allowable error) to find a clustering of the labels $\partition$. From this clustering, we can define the proposal distribution $\approxsoftmax(y|\datapoint_t,\partition)$ (Equation~\ref{eq:approxclustersoftmax}). Then, we use this proposal distribution to run a given number of Metropolis-Hastings steps to provide samples from the softmax distribution. The resulting samples are used to compute a loss for the given example, $\hat{\loss}(\datapoint_t,\lbl_t)$, and the result in a gradient step for the dual-encoder parameters. After every $w$ steps, we update the tree using Algorithm~\ref{alg:update_sg}. We describe the algorithm as SGD for simplicity; a minibatch version is used in practice.

\begin{algorithm}[h]
\caption{\ours Training Algorithm}
\begin{algorithmic}[1]
\STATE{\textbf{Input:} $\labelset$: Targets, $\dataset_\mathsf{train} = \{(\datapoint_1,\lbl_1),\dots,\}$: Training, $\gamma,m$: allowed  error, $\dataenc,\lblenc$: encoders, $k$: num samples, $\eta$: learning rate, $U$: update upper-bound} 
\STATE \textsc{InitializeTargetEncoding()} \mycomment{Algorithm~\ref{alg:initialEncoding}}
\STATE Construct $\hc$ \mycomment{Algorithm from \cite{zaheer2019sg}}
\FOR{$t$ from 0 to \textrm{NumSteps}}
\STATE Sample $(x_t,y_t)$ from $\dataset_\mathsf{train}$
\STATE $\partition \gets$ \textsc{FindClustering}$(\hc, x_t, \gamma, m)$ \mycomment{Alg.~\ref{alg:covertreesearchforpartition}}
\STATE $\approxsoftmax(\lbl|\datapoint;\partition) = \frac{1}{\hat{Z}}\exp(\invtemp\ip{\dataenc(\datapoint)}{\lblenc(\lbl^{(\partition)})})$ \mycomment{Eq.~\ref{eq:approxclustersoftmax}}
\STATE Sample $N_s = \{y_{s_i} : i \in [k]\}$ with  $y_{s_i}\sim Q_\mathsf{MH}$.

\STATE \resizebox{.99\linewidth}{!}{
$\hat{Z} \gets \exp( \invtemp\ip{\dataenc(\datapoint_t)}{\lblenc(\lbl_t)}) + \sum_{i \in [k]} \exp( \invtemp\ip{\dataenc(\datapoint_t)}{\lblenc(\lbl_{s_i})} )$
}
\STATE $\hat{\loss} \gets -\invtemp\ip{\dataenc(\datapoint_t)}{\lblenc(\lbl_t)} + \log(\hat{Z}) $.
\STATE $\Theta \gets \Theta - \eta \nabla_\Theta \hat{\loss}$
\IF{$t \ \textrm{mod} \ w = 0$}
\STATE \textsc{ApproximateReEncode()} \mycomment{Algorithm~\ref{alg:reEncode}}
\STATE $\hc \gets \textsc{UpdateSGTree}(\hc, \emptyset, U)$ \mycomment{Algorithm~\ref{alg:update_sg}}
\ENDIF
\ENDFOR
\STATE \textbf{return} $\dataenc,\lblenc$
\end{algorithmic}
\label{alg:dynnibalComplete}
\end{algorithm}

\begin{table*}
\vspace{-2mm}
    \centering
    \begin{tabular}{@{}l c c c c c c@{}}
    \toprule
         Performance  & Mem. \% & R@1 & R@5 & R@10 & R@20 & R@100 \\
         \midrule
         In-batch Negatives& 0.0\% &  0.356 &  0.613 &  0.695 &  0.757 &  0.843  \\
         Uniform Negatives& 0.0\% & 0.386 &  0.644 &  0.723 &  0.775 &  0.848  \\
         \rowcolor{red!10} \ours (This Paper) &  0.3\% &  0.485 &  0.695 &  0.754 &  0.801 &  0.862  \\
         Stochastic Negative Mining & 1.0\% & 0.444 &  0.672 &  0.739 &  0.785 &  0.855 \\
         \midrule
         Stochastic Negative Mining &  3\% & 0.461 &  0.689 &  0.750 &  0.794 &  0.860 \\
         Negative Cache & 6.3\% & - & - & - & 0.784 & 0.856 \\
         Stochastic Negative Mining & 10\% & 0.468 &  0.690 &  0.758 &  0.798 &  0.862 \\
         \midrule
         Exhaustive Brute Force & 100\% &  0.500 &  0.700 &  0.765 &  0.804 &  0.866 \\
    \bottomrule
    \end{tabular}
    \caption{\textbf{Natural Questions}. We report the performance on the test set using the answer string retrieval recall \citep{karpukhin2020dense}. Methods are listed in ascending order of required accelerator memory (as \% of Exhaustive Brute Force memory). Scaling to many targets requires an approach that does not need accelerator memory, since storing such targets on the accelerator will become a limiting bottleneck in terms of memory.  \ours uses lower dimensional representations in CPU memory, an efficient tree update, and approximate re-encoding during training. \ours performs significantly better than the practical approaches with $\leq1.0\%$ memory. Error with respect to the Exhaustive Brute Force method is cut by half in terms of R@1 with respect to all methods other than the impractical 10\% memory Stochastic Negative Mining. In more detail, \ours (0.485 R@1) cuts error to 1.5 points compared to Exhaustive Brute Force (0.500), compared to Stochastic Negative Mining 3\% (0.461) cuts error to 3.9 points. Compared to  Stochastic Negative Mining 10\% (0.468) which cuts error to 3.2, \ours cuts error by over 2x more. In terms of R@5, \ours also observes a strong 2 point gain over the low-memory variant of Stochastic Negative Mining (1.0\% Mem.)}
    \label{tab:nq_table}
    \vspace{-4mm}
\end{table*}

\begin{table}[!htb]
    \vspace{2mm}
    \centering
    \begin{tabular}{@{}l@{} c c c c @{}}
    \toprule
         MRR  & Mem. \% & @1 & @10 & @100 \\
         \midrule
         In-batch Negatives & 0 & 0.140 & 0.242 & 0.254 \\
         Uniform Negatives & 0 &  0.196  & 0.305 & 0.316 \\
         Negative Cache & 0.06\% & - & 0.310 & - \\
         Negative Cache & 0.24\% & - & 0.315 & - \\
         \rowcolor{red!10} \ours   &   0.76\%  &  0.223 &  0.334  & 0.345  \\
         Negative Cache & 0.96\%   & - & 0.323 & - \\
         Stochastic Neg. & 1.0\% &  0.200  & 0.309 & 0.320  \\
         \midrule
         Stochastic Neg. & 3.0\% &  0.216  & 0.331 & 0.342  \\
         Negative Cache & 3.8\%  & - & 0.322  & - \\
         Negative Cache & 15.15\% & -  & 0.331 & - \\
         \midrule
         Exhaustive & 100\% &  0.228 & 0.345 & 0.356  \\
    \bottomrule
    \end{tabular}
    \caption{\textbf{MSMarco Results}. \ours again outperforms competing methods that use similar low percentage memory requirements. In particular, \ours outperforms 1\% memory Stochastic Negative Mining in terms of MRR@1 by more than 2 points. Furthermore, we see that even when Negative Cache uses 15.15\% memory, a 150x increase compared to our approach, \ours still produces a higher MRR@10 result.}
    \label{tab:msmarco_table}
\end{table}

\vspace{-4mm}
\section{EXPERIMENTS}
\label{sec:exp}
\vspace{-2mm}

We compare our proposed approach, \ours, to multiple state-of-the-art methods for training dual encoders. We evaluate on two retrieval datasets and also the entity linking dataset, {Zeshel} \citep{logeswaran2019zero}, in  \S\ref{app:zeshel}.

\noindent \textbf{Natural Questions (NQ)} \citep{kwiatkowski2019natural} is a dataset for passage retrieval. The data points are natural language questions. The targets are passages from Wikipedia. There are over 21 million targets and about 60K training examples. We report results using the string-match-based recall evaluation that is used by previous work \citep{karpukhin2020dense,lindgren2021efficient}. 

\noindent \textbf{MSMARCO} \citep{bajaj2016ms} contains 8.8 million targets and 500K training examples. Data points are natural language questions and targets are passages extracted from web documents. The task is to provide the correct passage for a given question. Following previous work, we report  the mean reciprocal rank.

We compare the following methods with dual encoders as transformer \citep{vaswani2017attention} initialized from pretrained RoBERTa base \citep{liu2019roberta}. See details in \S\ref{app:empirical_details}.

\noindent\textbf{In-batch Negatives}. We approximate the softmax distribution by only using the positive target labels from within the batch as in previous work (e.g., \cite{henderson2017efficient})

\noindent\textbf{Uniform Negatives}. Sample targets uniformly at random from the collection of targets and use this to approximate the softmax distribution (e.g.,  \cite{karpukhin2020dense}). 

\noindent\textbf{Stochastic Negative Mining} \citep{reddi2019stochastic}. A large number of targets is sampled uniformly at random. These targets are stored with their stale representations on the accelerator device. From this large number of targets, we approximate the softmax distribution with $k$ hard negatives. We periodically, e.g. every 100 or 500 steps refresh and change the negatives stored on the accelerator device. The performance of this method depends on how many targets are stored on the accelerator device. In most settings, memory becomes a bottleneck before the computational burden of computing pairwise similarities on the accelerator. We report performance in terms of this memory bottleneck, relative to the overall dataset size. 

\noindent\textbf{Negative Cache} \citep{lindgren2021efficient}. A recent approach that keeps track of a collection of targets in a cache on accelerator similar to Stochastic Negative Mining. However, rather than randomly refreshing, negatives are added and removed the targets in a streaming manner, FIFO or LRU. We report results directly from the published paper. 

\noindent\textbf{Oracle Exhaustive Brute-Force}. We exhaustively compute all logits and exhaustively find the top-$k$ closest targets for training. Including this method provides insights into upper-bound empirical results, but it is impractical 
in practice since it is 
extremely expensive in computation and accelerator memory, and thus also financial cost. 
While we obtain results from this ``oracle'' method on the datasets above, running on meaningfully larger data would not have been possible even if computation/budget were no object.

\noindent\textbf{\ours}. 
This paper's proposed cluster and approximate re-encoding approach. We note that the only portion of our training method sitting on accelerator memory is the low-rank regression-based approximate re-encoding model (\S\ref{sec:approxReencode}), the landmark points, and regression training data. Our use of lower dimensional (64 or 128) \nystrom embeddings along with scalable and relatively lightweight index structures stored in cheap CPU memory would allow our approach to scale to billions of targets, and further scale to many billions by leveraging the tree structure for additional targeted truncations and cluster-based approximations.  

Empirically, we find that the Metropolis-Hasting-based sampling outperforms rejection sampling. We find using lower temperature and approximating the sampling procedure to ensure that ``harder'' negatives are selected is beneficial to end task performance. Rather than considering the entire partition of targets, $\partition$, returned by Algorithm~\ref{alg:covertreesearchforpartition}, we consider the top-$k$ closest clusters. We set $m=-11$ (MSMARCO) and $m=-14$ (NQ) as the deepest level of clusters and restrict the size of frontier in Algorithm~\ref{alg:covertreesearchforpartition} to be 100. We run chains of length 2. Finally, as another empirical approximation, we select the top scoring targets from all the chains. We discuss these choices further in \S\ref{app:empirical_details}.

\subsection{Empirical Results}
\vspace{-4mm}
We report the performance of each method in terms of each dataset's given evaluation metric. Alongside the performance, we report the accelerator (GPU/TPU) memory requirement as a percentage with respect to the Oracle Exhaustive Brute-Force approach (storing all targets as full dimensional embeddings). 
As noted when describing Stochastic Negative Mining, we find that the main bottleneck for our competitors is the memory required to store target embeddings on the accelerators (along with the transformer encoders), not the computation of logits.
Thus of crucial importance are \ours's advantages over baselines with respect to reduced need for accelerator memory.

Table~\ref{tab:nq_table} shows the results on the test set for NQ. We observe that \ours outperforms In-batch Negatives, Uniform Negatives, and Stochastic Negative Mining using 1.0\% memory in all recall metrics.  In terms of recall at 1, \ours cuts the performance gap with respect to exhaustive brute force by more than half compared to all methods except 10\% memory Stochastic Negative Mining. Even using 10\% memory with Stochastic Negative Mining compared to 0.3\% for \ours leaves large performance gaps in terms of recall at 1. Stochastic Negative Mining achieves 0.468 versus \ours's 0.485.

Table~\ref{tab:msmarco_table} reports performance for all methods on MSMARCO in terms of mean reciprocal rank. Following past work \citep{lindgren2021efficient}, we evaluate on the development set. We find that performance on MSMARCO follows a similar trend as NQ in that \ours outperforms all of the other low-memory approaches. In fact, even using 150x more memory Negative Cache still does not perform as well as \ours. \ours sees a multiple point improvement in MRR@1 compared to competing methods. 

We consider the time needed to maintain the SG tree structure during training. We find on MSMARCO that, given a collection of updated target embeddings, our dynamic method (\S\ref{sec:dynamic}) is about 4x faster than the traditional approaches that rebuild the entire structure. Furthermore, we find that our regression-based approximate re-encoding is about 8x faster than running the encoder model on MSMARCO, and requires much less accelerator expense. In terms of training steps-per-second efficiency between index maintenance, our approach does involve more computation: it is about four times slower than Stochastic Negative Mining's simple logit-based selection on a small subset of targets; however, some of that is balanced by faster index updates. We expect that additional engineering design could significantly speed up our method; in any case, running Stochastic Negative Mining for more training steps did not increase its accuracy.
\section{RELATED WORK}

\textbf{Per-Target Free Parameters}. 
A large body of work trains classifiers with a massive number of targets. In much of this work, the targets are directly parameterized with a $d$ dimensional vector, rather an embedding produced by an encoder \citep{bengio2008adaptive,choromanska2015logarithmic,jernite2017simultaneous,daume2017logarithmic,sun2019contextual,yu2020pecos}. Most closely related to our work are methods that adaptively re-arrange tree-structured classifiers \cite[inter alia]{sun2019contextual,kobus21aOnline}.
\cite{blanc2018adaptive,rawat2019sampled} use kernel-based approximations of the softmax, which gives strong theoretical guarantees. Future work could consider how to extend these methods to the dual-encoder setting.

\textbf{Partition Functions \& Probabilistic Models.} 
MCMC methods are widely used for posterior inference \citep[inter alia]{
neal1993probabilistic,neal2000markov,chang2013parallel,zaheer2016exponential,zaheer2017canopy}. These methods are concerned with finding a high quality estimate of the distribution as the end task. In our setting, we need an estimate for every training point in each step of training. Most similar (and the inspiration for our approach) is the Canopy-sampler \citep{zaheer2017canopy}, which uses a cover tree to derive an efficient rejection sampler/Metropolis-Hastings algorithm for exponential family mixture models. Unlike that work, we consider approximation of the softmax, moving embedded representations during training, and the relationship between the approximation quality and the gradient bias of our estimator. \cite{vembu2009probabilistic} use MCMC-based approximations of distributions with large output spaces, but does not use the clustering-based approximations presented here nor does it consider the dual-encoder setting.

\textbf{Reparameterization, discrete distributions}. OS* sampling, the Gumbel-Max trick, and Perturb-and-MAP, are alternative methods for sampling from distributions such as the softmax  \citep{dymetman2012algorithm,tucker2017rebar,maddison2016concrete,paulus2020gradient,jang2016categorical, huijben2022review}. Future work could consider combining such methods and our approach. 

\textbf{Nearest Neighbor Search, Clustering, Dynamic Structures}. Cover Trees \citep{beygelzimer2006cover} and SG Trees \citep{zaheer2019sg} were originally used as nearest neighbor indexes. There are many tree and DAG-based index structures that support addition and deletion of items such as Navigating Nets \citep{krauthgamer2004navigating},  HNSW \citep{malkov2014approximate}, and NN-Descent \citep{dong2011efficient}. Also, closely related is work on maintaining dynamic hash-based indexes \citep{jain2008online,zhang2020continuously}. Apart from nearest neighbor index structures, scalable methods for hierarchical clustering organize large datasets into tree structures \citep{bateni2017affinity,moseley2019framework,dhulipala2022hierarchical}. Other work builds these clusterings in an incremental or dynamic setting \citep{liberty2016algorithm,choromanska2012online,kobren2017hierarchical,menon2019online}. Maintaining structures in a dynamic setting is studied in graph algorithms (e.g., minimum spanning tree).   Efficient and parallelizable dynamic graph algorithms exist \cite[inter alia]{Sleator1981ADS,holm2001poly,tseng2019batch,dhulipala2020parallel}.

\noindent \textbf{Cover Trees.} New algorithms and extensions of cover trees have been developed \cite[inter alia]{curtin2013tree,curtin2015plug,izbicki2015faster,zaheer2019sg,elkin2021new,gu2022parallel}. Cover trees are also widely used in other settings such as $k$-means clustering \citep{curtin2017dual} and Gaussian processes \citep{terenin2022numerically}. 

\noindent \textbf{Task Specific Related Work.} Learned models for passage retrieval and entity linking are extensively studied \cite[inter alia]{wu2019scalable,karpukhin2020dense,bhowmik2021fast,qu2021rocketqa,thakur2021beir,ren2021rocketqav2,ni2021large,fitzgerald-etal-2021-moleman,gao2021condenser,dai2022promptagator,izacard2022few}. Most similar to our work is ANCE \citep{xiong2020approximate}, which uses a nearest neighbor index for a contrastive objective. Other work includes alternative architectures to dual encoders \citep{khattab2020colbert,qian2022multi,santhanam2021colbertv2} and learning efficient hashing-based representations \citep{yamada2021efficient}.  
\section{CONCLUSION}
We present \ours, a dynamic tree structure for clustering-based approximations of the softmax distribution. Our algorithm efficiently gives provably accurate samples for training dual-encoders with cross-entropy loss. 
Empirically, \ours outperforms state-of-the-art on datasets with over twenty million targets, reducing error by half compared to an exhaustive oracle. 
We find that our dynamic maintenance of the tree structure can be 8x faster than exhaustive re-indexing. Furthermore, our approach outperforms state-of-the-art while using 150x less accelerator memory.

\bibliography{custom}
\bibliographystyle{plainnat}

\appendix
\onecolumn
\begin{centering}
\Large{\textbf{Supplemental Material: Improving Dual-Encoder Training through Dynamic Indexes for Negative Mining}}
\end{centering}
\section{EFFICIENT AND ACCURATE SAMPLES FROM THE SOFTMAX DISTRIBUTION}

\subsection{Rejection Sampling}
\label{sec:rejection_sampling}

Rejection sampling for the softmax distribution follows closely the proof in \cite{zaheer2017canopy} for mixture models. 

We sample from: 
\begin{align*}
    \lbl \sim \frac{\exp(\invtemp\ip{\dataenc(\datapoint)}{\lblenc(\hat{\lbl})})}{\sum_{\lbl'}\exp(\invtemp\ip{\dataenc(\datapoint)}{\lblenc(\hat{\lbl'})})}
\end{align*}
and accept with probability
\begin{align*}
  \expNegRejError \frac{\exp(\invtemp\ip{\dataenc(\datapoint)}{\lblenc(\lbl)})}{\exp(\invtemp\ip{\dataenc(\datapoint)}{\lblenc(\hat{\lbl})})}.
\end{align*}

If we have
\begin{align}
     \expNegRejError \leq \frac{\exp(\invtemp\ip{\dataenc(\datapoint)}{\lblenc(\hat{\lbl})})}{\exp(\invtemp\ip{\dataenc(\datapoint)}{\lblenc(\lbl)})} \leq \expRejError.
 \end{align}
Then, if we want determine the probability of sampling a particular target, $\lbl$, denoted $\textsf{Pr}(\lbl)$. Producing $\lbl$ can be done by \textcolor{teal}{ sampling and accepting $\lbl$ or} and \textcolor{purple}{sampling and rejecting another target $\lbl'$ and then accepting $\lbl$ in one of the subsequent rounds of sampling}.
\newcommand{\acceptY}{{\color{teal} \frac{\exp(\invtemp\ip{\dataenc(\datapoint)}{\lblenc(\hat{\lbl})})}{\sum_{\lbl''}\exp(\invtemp\ip{\dataenc(\datapoint)}{\lblenc(\hat{\lbl''})})} \expNegRejError \frac{\exp(\invtemp\ip{\dataenc(\datapoint)}{\lblenc(\lbl)})}{\exp(\invtemp\ip{\dataenc(\datapoint)}{\lblenc(\hat{\lbl})})} }}
\newcommand{\rejectOther}{\color{purple} \textsf{Pr}(\lbl) \sum_{y' \in \labelset} \left ( 1 - \expNegRejError \frac{\exp(\invtemp\ip{\dataenc(\datapoint)}{\lblenc(\lbl')})}{\exp(\invtemp\ip{\dataenc(\datapoint)}{\lblenc(\hat{\lbl'})})} \right ) 
\frac{\exp(\invtemp\ip{\dataenc(\datapoint)}{\lblenc(\hat{\lbl'})})}{\sum_{\lbl''}\exp(\invtemp\ip{\dataenc(\datapoint)}{\lblenc(\hat{\lbl''})})}
}
\begin{equation}
\resizebox{.99\linewidth}{!}{%
$\begin{aligned}
    \textsf{Pr}(\lbl) &= \acceptY + \rejectOther \\
    &= \expNegRejError \frac{ \exp(\invtemp\ip{\dataenc(\datapoint)}{\lblenc({\lbl})})}{\sum_{\lbl''}\exp(\invtemp\ip{\dataenc(\datapoint)}{\lblenc(\hat{\lbl''})})} + \frac{\textsf{Pr}(\lbl)}{\sum_{\lbl''}\exp(\invtemp\ip{\dataenc(\datapoint)}{\lblenc(\hat{\lbl''})})} \sum_{y' \in \labelset} \left ( 1 - \expNegRejError \frac{\exp(\invtemp\ip{\dataenc(\datapoint)}{\lblenc(\lbl')})}{\exp(\invtemp\ip{\dataenc(\datapoint)}{\lblenc(\hat{\lbl'})})} \right ) \exp(\invtemp\ip{\dataenc(\datapoint)}{\lblenc(\hat{\lbl'})}) \\
    &= \expNegRejError \frac{ \exp(\invtemp\ip{\dataenc(\datapoint)}{\lblenc({\lbl})})}{\sum_{\lbl''}\exp(\invtemp\ip{\dataenc(\datapoint)}{\lblenc(\hat{\lbl''})})} + \frac{\textsf{Pr}(\lbl)}{\sum_{\lbl''}\exp(\invtemp\ip{\dataenc(\datapoint)}{\lblenc(\hat{\lbl''})})} \sum_{y' \in \labelset} \left ( \exp(\invtemp\ip{\dataenc(\datapoint)}{\lblenc(\hat{\lbl'})}) - \expNegRejError {\exp(\invtemp\ip{\dataenc(\datapoint)}{\lblenc(\lbl')})} \right )  \\
    &= \expNegRejError \frac{ \exp(\invtemp\ip{\dataenc(\datapoint)}{\lblenc({\lbl})})}{\sum_{\lbl''}\exp(\invtemp\ip{\dataenc(\datapoint)}{\lblenc(\hat{\lbl''})})} + \frac{\textsf{Pr}(\lbl)}{\sum_{\lbl''}\exp(\invtemp\ip{\dataenc(\datapoint)}{\lblenc(\hat{\lbl''})})} \left ( \sum_{y' \in \labelset}  \exp(\invtemp\ip{\dataenc(\datapoint)}{\lblenc(\hat{\lbl'})}) - \expNegRejError \sum_{y' \in \labelset}  {\exp(\invtemp\ip{\dataenc(\datapoint)}{\lblenc(\lbl')})} \right )  \\
    &= \expNegRejError \frac{ \exp(\invtemp\ip{\dataenc(\datapoint)}{\lblenc({\lbl})})}{\sum_{\lbl''}\exp(\invtemp\ip{\dataenc(\datapoint)}{\lblenc(\hat{\lbl''})})} + \textsf{Pr}(\lbl) - \frac{\textsf{Pr}(\lbl)}{\sum_{\lbl''}\exp(\invtemp\ip{\dataenc(\datapoint)}{\lblenc(\hat{\lbl''})})} \expNegRejError Z  \\
    \textsf{Pr}(\lbl) &= \frac{1}{Z} \exp(\invtemp\ip{\dataenc(\datapoint)}{\lblenc({\lbl})})
\end{aligned}$}
\end{equation} 
And so the rejection sampling strategy will sample from the true softmax distribution.

\subsection{Metropolis-Hastings Approximate Softmax}
\label{sec:mh_appendix}

A common approach for providing samples from a distribution that is difficult to sample from is the Metropolis-Hastings algorithm. 
Recall that Metropolis-Hastings produces a sample from $P(\lbl|\datapoint)$ by iteratively sampling a state change from a proposal distribution, denoted $\approxsoftmax$, and determines whether or not to `accept' the state change based on an \emph{acceptance ratio}.
In particular, we will use an \emph{independent} Metropolis-Hastings method. That is the proposal distribution $\approxsoftmax$ is independent of the current state of the Markov chain. For the $t$-state in the chain, we have a proposal distribution of the form $\lbl^{(t)} \sim \approxsoftmax(\lbl|\datapoint,\lbl^{(t-1)}) = \approxsoftmax(\lbl|\datapoint)$. 
The acceptance ratio is defined as:
\begin{equation}
    \acceptanceratio(\lbl^{(t-1)}, \lbl^{(t)}) = {\min} \left (1, \frac{P(\lbl^{(t)}|\datapoint)}{P(\lbl^{(t-1)}|\datapoint)}
    \frac{\approxsoftmax(\lbl^{(t-1)}|\datapoint)}{\approxsoftmax(\lbl^{(t)}|\datapoint)} \right )
\end{equation} 

\subsection{Proofs for \S\ref{sec:search}: Hierarchical Clustering Structures}

\unnormalizedError*

\begin{proof}

\begin{align}
   \frac{\exp(\beta \ip{\dataenc(\datapoint)}{\lblenc(\lbl^{(\partition)}})}{\exp(\beta\ip{\dataenc(\datapoint)}{\lblenc(\lbl)})} 
    &\leq \exp(\beta \ip{\dataenc(\datapoint)}{\lblenc(\lbl^{(\partition)}) - \lblenc(\lbl)}) \\ 
     &\leq \exp(\beta \norm{\dataenc(\datapoint)}_2 \norm{\lblenc(\lbl^{(\partition)}) - \lblenc(\lbl)}_2) &\text{\mycomment{Cauchy-Schwartz}}\\  \\ 
     &\leq \exp(\beta \norm{\lblenc(\lbl^{(\partition)}) - \lblenc(\lbl)}_2) &\text{\mycomment{Assumption~\ref{assumption:unitNorm}}}\\ 
     &\leq \exp(\beta \ctbase^\ell) &\text{\mycomment{Covering Property}}. 
\end{align}

\end{proof}

\normalizedError*
\begin{proof}

Let $\ell$ be the level of the clustering $\partition$ used.

\begin{align}
    \max_{\lbl \in \labelset} \frac{P(y|x)}{Q(y|x;\partition_{(\ell)})} &= \max_{\lbl \in \labelset}  \frac{\exp(\beta \ip{\dataenc(\datapoint)}{\lblenc(\lbl})}{\exp(\beta\ip{\dataenc(\datapoint)}{\lblenc(\lbl^{(\partition_{(\ell)})})}} \cdot \frac{\hat{Z}}{Z} \\
    & \leq  \max_{\lbl \in \labelset} \frac{\exp(\beta \ip{\dataenc(\datapoint)}{\lblenc(\lbl})}{\exp(\beta\ip{\dataenc(\datapoint)}{\lblenc(\lbl^{(\partition_{(\ell)})})}} \cdot \max_{\lbl' \in \labelset}  \frac{\exp(\beta \ip{\dataenc(\datapoint)}{\lblenc(\lbl'^{(\partition_{(\ell)})})})}{\exp(\beta \ip{\dataenc(\datapoint)}{\lblenc(\lbl'})} 
\end{align}
The above inequality follows by property of mediant. Now define:
\begin{align}
    R_1 & \triangleq \max_{\lbl \in \labelset} 
    \frac{
        \exp(\beta \ip{\dataenc(\datapoint)}{\lblenc(\lbl)})
    }
    {
        \exp(\beta\ip{\dataenc(\datapoint)}{\lblenc(\lbl^{(\partition_{(\ell)})})})
    }  \\
    R_2 & \triangleq  \max_{\lbl' \in \labelset}  \frac{
        \exp(\beta\ip{\dataenc(\datapoint)}{\lblenc({\lbl'}^{(\partition_{(\ell)})})})
    }{
        \exp(\beta \ip{\dataenc(\datapoint)}{\lblenc(\lbl')}))
    }  .
\end{align}

Observe for $R_1$:
\begin{align}
    \max_{\lbl \in \labelset} 
    \frac{
        \exp(\beta \ip{\dataenc(\datapoint)}{\lblenc(\lbl)})
    }
    {
        \exp(\beta\ip{\dataenc(\datapoint)}{\lblenc(\lbl^{(\partition_{(\ell)})})})
    }   &=  
    \max_{\lbl \in \labelset} {
        \exp(\beta \ip{
            \dataenc(\datapoint)}
            {\lblenc(\lbl)} - \beta\ip{\dataenc(\datapoint)}{\lblenc(\lbl^{(\partition_{(\ell)})})})
    } \\
    &= \max_{\lbl \in \labelset} {\exp(\beta \ip{\dataenc(\datapoint)}{\lblenc(\lbl) -   \lblenc(\lbl^{(\partition_{(\ell)})})})} \\
    &\leq \max_{\lbl \in \labelset} \exp(\beta \norm{\dataenc(\datapoint)}_2 \norm{\lblenc(\lbl) -   \lblenc(\lbl^{(\partition_{(\ell)})})}_2) \\
    &\leq \max_{\lbl \in \labelset} \exp(\beta \norm{\lblenc(\lbl) -   \lblenc(\lbl^{(\partition_{(\ell)})})}_2)
    \leq \exp(\beta \ctbase^{\ell})
\end{align}
Similarly for $R_2$:
\begin{align}
    \max_{\lbl \in \labelset}  \frac{\exp(\beta \ip{\dataenc(\datapoint)}{\lblenc(\lbl^{(\partition_{(\ell)})}})}{\exp(\beta\ip{\dataenc(\datapoint)}{\lblenc(\lbl)}} 
    \leq \max_{\lbl \in \labelset} \exp(\beta \norm{\lblenc(\lbl^{(\partition_{(\ell)})}) -   \lblenc(\lbl)}_2)
    \leq \exp(\beta \ctbase^{\ell})
\end{align}

Combining above two:
\begin{align}
    \max_{\lbl \in \labelset} \frac{P(y|x)}{Q(y|x;\partition)} \leq R_1 \times R_2 \leq  \exp(\beta \ctbase^{\ell}) \exp(\beta \ctbase^{\ell}) = \exp(2\beta \ctbase^{\ell})
\end{align}

Thus, to maintain maximum ratio $\gamma$, we need to select $\ell$ such that:
\begin{align}
    \exp(2\beta \ctbase^{\ell}) & \leq \gamma \\
    \ctbase^\ell &\leq \frac{1}{2\beta} \log \gamma 
\end{align}
which is as stated in the proposition.
\end{proof}

\changeFlat*
\begin{proof}
Let,
\begin{align}
    \ctbase^\ell &\leq \frac{1}{2\beta} \log \gamma \\
    \ctbase^{\ell-1} &\leq \frac{1}{2\beta} \log {\gamma'},
\end{align}
then,
\begin{align}
    \frac{1}{b}\ctbase^\ell &\leq \frac{1}{b} \frac{1}{2\beta} \log \gamma \\
    \ctbase^{\ell-1} &\leq \frac{1}{2\beta} \log {\gamma}^{\frac{1}{b}},
\end{align}
and so $\gamma'\leq \gamma^{\frac{1}{b}}$.
\end{proof}

\rejectionSamplingThm*
\begin{proof}
We want to show two properties of the rejection sampling algorithm, its running time and its correctness (e.g., that it samples from the true posterior).

The proof follows very closely to the analogous algorithm for mixture models by \cite{zaheer2017canopy}.

First, we want to show that the running time is $\bigo(|\partition_{(\ell)}| + \textsf{BF} e^{\ctbase^\ell})$ where $\textsf{BF}$ is the branching factor for cover/SG trees ($\bigo(\alpha^4)$ and $\bigo(\alpha^3)$ respectively). We begin by considering the expected number of rejections from the first level $\ell$. Based on our definition of $Q$, the number of samples is upper bounded by $e^{\beta \ctbase^{\ell+2}}$, therefore the number of rejections is $e^{\beta \ctbase^{\ell+2}}-1$. If we consider the next level down, there would be $e^{\beta \ctbase^{\ell+1}}-1$ rejections, the level after that $e^{\beta \ctbase^{\ell}}-1$, $e^{\beta \ctbase^{\ell-1}}-1$, etc. We can use the branching factor $\textsf{BF}$ to give a bound on how expensive the sampling step is at every level. We are interested therefore in:
\begin{align}
   \textsf{BF} \sum_{k=1}^\infty \left (e^{\beta \ctbase^{\ell-k}}-1 \right ).
\end{align}
We have that $e^x -1 \leq xe^a$ for $x\in[0,a]$ and $\sum_{k=1}^\infty \ctbase^{-k} = 1$ and so: 
\begin{align}
   \textsf{BF} \sum_{k=1}^\infty \left (e^{\beta \ctbase^{\ell-k}}-1 \right ) \leq \textsf{BF} \cdot e^{\beta \ctbase^{\ell}} \sum_{k=1}^\infty \ctbase^{-k} = \textsf{BF} \cdot e^{\beta\ctbase^{\ell}}
\end{align}
We then need to consider the cost of the initial sampling step, which is not $\textsf{BF}$, but rather depends on the number of clusters in the partition $|\partition_{(\ell)}|$, leading to $\bigo(|\partition_{(\ell)}| + \textsf{BF}\cdot e^{\beta \ctbase^\ell})$. It is important to note here that while the running time of the algorithm depends only on the branching factor of the trees, the depth of the SG tree differs from that of the cover tree. The depth of the SG tree is $\bigo(\log(\frac{d_\textrm{max}}{d_\textrm{min}}))$ where $d_\textrm{max}$ is the largest pairwise distance and $d_\textrm{min}$ the minimum pairwise distance (this ratio also know as the aspect ratio), whereas the depth of the cover tree is $\bigo(\alpha^2\log |\labelset|)$. 

Next, let's consider the correctness of the method. We want to determine the probability of sampling a particular target $\lbl$, denoted $\textsf{Pr}(y)$. Recall that to sample $y$, we need to follow the path in the tree from level $\ell$ to the level in which $y$ first appears as a cluster representative and accept that cluster. We will refer to this level in which $y$ first appears as $k$. There is a path of selected clusters/nodes $\cluster_\ell,\cluster_{\ell-1},\dots,\cluster_k,$, where we indicate the level of the selected node in the subscript. Let $\mathcal{A}(\lbl)$ be the probability of accepting $\lbl$. To reduce the complexity of notation, define: 
\begin{align}
    \textsf{w}_{\cluster} = \exp(\invtemp\ip{\dataenc(\datapoint)}{\lblenc(\cluster)})
\end{align}
We can write $\mathcal{A}(\lbl)$ as:
\begin{align}
    \mathcal{A}(\lbl) &= \frac{1}{{Z}_\ell} e^{\beta \ctbase^\ell} \blacktriangle_{\cluster_\ell} \textsf{w}_{\cluster_\ell} &\text{\mycomment{Sample at top level}} \\
    & \times \left [ \prod_{j=k+1}^\ell \left (1-\frac{\frac{1}{|\labelset|} \textsf{w}_{\cluster_j}}{e^{\beta\ctbase^j} \blacktriangle_{\cluster_j} \textsf{w}_{\cluster_j}} \right ) \frac{e^{\beta\ctbase^{j-1}}\blacktriangle_{\cluster_{j-1}} \textsf{w}_{\cluster_{j-1}}}{e^{\beta\ctbase^j} \blacktriangle_{\cluster_{j}} \textsf{w}_{\cluster_j} - \frac{1}{|\labelset|} \textsf{w}_{\cluster_j} } \right ] &\text{\mycomment{Rejecting and selecting path to $\cluster_k$}} \\
    &\times \frac{\frac{1}{|\labelset|}\textsf{w}_{\cluster_k}}{e^{\beta\ctbase^k \blacktriangle_{\cluster_{k}} \textsf{w}_{\cluster_k}}} &\text{\mycomment{Accepting the given cluster $\cluster_k$}} \\
    &= \frac{1}{{Z}_\ell} e^{\beta \ctbase^\ell} \blacktriangle_{\cluster_\ell} \textsf{w}_{\cluster_\ell} \left [\prod_{j=k+1}^\ell \frac{e^{\beta\ctbase^{j-1}} \blacktriangle_{\cluster_{j-1}} \textsf{w}_{\cluster_{j-1}}}{e^{\beta\ctbase^{j}} \blacktriangle_{\cluster_{j}} \textsf{w}_{\cluster_{j}}} \right ] \frac{\frac{1}{|\labelset|} \textsf{w}_{\cluster_{k}}}{e^{\beta\ctbase^{k}} \blacktriangle_{\cluster_{k}} \textsf{w}_{\cluster_{k}}} \\
    &= \frac{1}{{Z}_\ell} e^{\beta \ctbase^\ell} \blacktriangle_{\cluster_\ell} \textsf{w}_{\cluster_\ell} \left [ \frac{e^{\beta\ctbase^{k}} \blacktriangle_{\cluster_{k}} \textsf{w}_{\cluster_{k}}}{e^{\beta\ctbase^{\ell}} \blacktriangle_{\cluster_{\ell}} \textsf{w}_{\cluster_{\ell}}} \right ] \frac{\frac{1}{|\labelset|} \textsf{w}_{\cluster_{k}}}{e^{\beta\ctbase^{k}} \blacktriangle_{\cluster_{k}} \textsf{w}_{\cluster_{k}}} \\
    &= \frac{1}{{Z}_\ell} \frac{1}{|\labelset|} \textsf{w}_{\cluster_{k}}
\end{align}
Now, we have to consider the probability $\mathcal{R}$ of restarting the sampler, e.g., 
\begin{align}
    \mathcal{R} = 1-\sum_{\lbl'}\mathcal{A}(\lbl') = 1 - \sum_{\lbl'} \frac{1}{{Z}_\ell} \frac{1}{|\labelset|} \textsf{w}_{\cluster_{k}} = 1 - \frac{1}{|\labelset|}\frac{Z}{Z_\ell}
\end{align}
Finally, consider $\textsf{Pr}(\lbl)$:
\begin{align}
    \textsf{Pr}(\lbl) &= \mathcal{A}(\lbl) +  \textsf{Pr}(\lbl) \cdot \mathcal{R} \\
    &= \frac{1}{{Z}_\ell} \frac{1}{|\labelset|} \textsf{w}_{\lbl} + \textsf{Pr}(\lbl)  - \textsf{Pr}(\lbl)  \frac{1}{|\labelset|}\frac{Z}{Z_\ell} \\
    \textsf{Pr}(\lbl) &= \frac{1}{Z}  \textsf{w}_{\lbl}
\end{align}
Therefore, the algorithm samples from the true softmax.
\end{proof}

\searchCorrectness*
\begin{proof}
Recall that Proposition~\ref{proposition:normalizedError} ensures that the initial partition achieves $\max_{\lbl \in \labelset} \frac{P(\lbl|\datapoint)}{Q(\lbl|\datapoint;\partition_{(\ell)})} = \gamma$. For any cluster in the partition discovered by the algorithm, $\cluster \in \partition$, let $k < \ell$ be level of the cluster. We therefore have: 
\begin{align}
    \max_{\lbl \in \cluster} \frac{P(\lbl|\datapoint)}{Q(\lbl|\datapoint;\partition)} \leq \exp(\beta \ctbase^k) \leq \exp(\beta \ctbase^\ell).  
\end{align}

The covering property and triangle inequality ensures that if $\norm{\dataenc(\datapoint) - \lblenc(\cluster)}_2 > \ctbase^k + \ctbase^{m}$, then every target in $\cluster$ at least $\ctbase^m$ from ${\dataenc(\datapoint)}$.

\end{proof}

\subsection{Proofs for \S\ref{sec:covertree_estimaor}: Gradient-Bias of Our Estimator}

\gradientBias*
\begin{proof}

Observe that:
\begin{align}
    \nabla_\Theta {\loss}(\datapoint_i,\lbl_i) &= -\invtemp \nabla_\Theta \ip{\dataenc(\datapoint_i)}{\lblenc(\lbl_i)} + \invtemp \E_{P}[\nabla_\Theta \ip{\dataenc(\datapoint_i)}{\lblenc(\lbl)}] \\ 
    \E[\nabla_\Theta {\hat{\loss}}(\datapoint_i,\lbl_i)] &= -\invtemp \nabla_\Theta \ip{\dataenc(\datapoint_i)}{\lblenc(\lbl_i)} + \invtemp \E_{Q_\textsf{MH}}[\nabla_\Theta \ip{\dataenc(\datapoint_i)}{\lblenc(\lbl)}] \\
    \E[\nabla_\Theta {\hat{\loss}}(\datapoint_i,\lbl_i)] - \nabla_\Theta {{\loss}}(\datapoint_i,\lbl_i) &= \beta \E_{Q_\textsf{MH}}[\nabla_\Theta \ip{\dataenc(\datapoint_i)}{\lblenc(\lbl)}] - \beta \E_{P}[\nabla_\Theta \ip{\dataenc(\datapoint_i)}{\lblenc(\lbl)}] \\
    \E[\nabla_\Theta {\hat{\loss}}(\datapoint_i,\lbl_i)] - \nabla_\Theta {{\loss}}(\datapoint_i,\lbl_i) &=\invtemp \sum_{y} \nabla_\Theta \ip{\dataenc(\datapoint_i)}{\lblenc(\lbl)} (P(y|x) - Q_\textsf{MH}(y|x)) 
\end{align}

Now using the bound on the total variation:
\begin{align}
    \norm{\E[\nabla_\Theta {\hat{\loss}}(\datapoint_i,\lbl_i)] - \nabla_\Theta {{\loss}}(\datapoint_i,\lbl_i)} &= \bignorm{\invtemp \sum_{y} \nabla_\Theta \ip{\dataenc(\datapoint_i)}{\lblenc(\lbl)} (P(y|x) - Q_\textsf{MH}(y|x))} \\ 
    &= \invtemp \sum_{y}\norm{ \nabla_\Theta \ip{\dataenc(\datapoint_i)}{\lblenc(\lbl)}} \cdot |(P(y|x) - Q_\textsf{MH}(y|x)| \\ 
    &\leq \invtemp \sum_{y} M |(P(y|x) - Q_\textsf{MH}(y|x)| \\ 
    &= \invtemp M \|P - Q_\textsf{MH}\|_1 \\ 
    &= 2\invtemp M \|P - Q_\textsf{MH}\|_\text{TV} \\
    &\leq 2 \invtemp M \epsilon 
\end{align}
\end{proof}

\subsection{Proofs for \S\ref{sec:dynamic}: Dynamic Maintenance of the Tree Structure}

\textbf{Proposition} \ref{proposition:pairwiseAmountOfChange} \emph{Under Assumptions~\ref{assumption:lipchitz},\ref{assumption:boundedGradient},\ref{assumption:unitNorm},
let $\phi_t$ and $\phi_{t+w}$ refer to the model parameters and model parameters after $w$ more steps of gradient descent with learning rate $\eta$.}
\begin{equation}
   \left | \norm{\lblencPTstep{t} - f_{\phi_{t}}(\lbl')}_2 - \norm{\lblencPTstep{t+w} - f_{\phi_{t+w}}(\lbl')}_2 \right | \leq 4w\eta \beta LM.
    \label{eq:changedAmount}
\end{equation}

\begin{proof}
We analyze this with similar techniques and assumptions as \cite{lindgren2021efficient}.
First consider a bound on the gradient norm. Let $k$ by the number of negative samples where the negative samples are given by $N_s = \{y_{s_j} : j \in [k]\}$
\begin{align}
   \norm{\nabla_\Theta {\hat{\loss}}(\datapoint_i,\lbl_i)} &= \bignorm{-\invtemp \nabla_\Theta \ip{\dataenc(\datapoint_i)}{\lblenc(\lbl_i)} + \invtemp \frac{1}{k} \sum_{j=1}^{k} \nabla_\Theta \ip{\dataenc(\datapoint_i)}{\lblenc(\lbl_{y_{s_j}})}} \\ 
   &\leq \invtemp \norm{\nabla_\Theta \ip{\dataenc(\datapoint_i)}{\lblenc(\lbl_i)}} + \invtemp \frac{1}{k} \bignorm{ \sum_{j=1}^{k} \nabla_\Theta \ip{\dataenc(\datapoint_i)}{\lblenc(\lbl_{y_{s_j}})}} \\
   &\leq \invtemp \norm{\nabla_\Theta \ip{\dataenc(\datapoint_i)}{\lblenc(\lbl_i)}} + \invtemp \frac{1}{k}  \sum_{j=1}^{k}\bignorm{ \nabla_\Theta \ip{\dataenc(\datapoint_i)}{\lblenc(\lbl_{y_{s_j}})}} \\
   &\leq 2\invtemp  M
\end{align}

Now, consider the difference of the dual encoder parameters at timestep, $t$, $\allparameters_{t}$ and at timestep, $t+w$, $\allparameters_{t+w}$.
\begin{align}
    \allparameters_{t} - \allparameters_{t+w} &= \sum_{i=t}^{t+w} \eta \nabla_\allparameters \hat{\loss}(x_i,y_i) \\
    \norm{\allparameters_{t} - \allparameters_{t+w}} &= \bignorm{\sum_{i=t}^{t+w} \eta \nabla_\allparameters \hat{\loss}(x_i,y_i)} \\
     &\leq \sum_{i=t}^{t+w} \eta \bignorm{\nabla_\allparameters \hat{\loss}(x_i,y_i)} \\
    &= w \eta 2 \invtemp M
\end{align}

Applying the triangle inequality
\begin{align}
    &\left | \norm{\lblencPTstep{t} - f_{\phi_{t}}(y')} - \norm{\lblencPTstep{t+w} - f_{\phi_{t+w}}(y')} \right | \\
    &\leq \left | \norm{\lblencPTstep{t} - f_{\phi_{t+w}}(y)} +  \norm{\lblencPTstep{t+w} - f_{\phi_{t}}(y')} - \norm{\lblencPTstep{t+w} - f_{\phi_{t+w}}(y')} \right | \\
    &\leq \left | \norm{\lblencPTstep{t} - \lblencPTstep{t+w}} +  \norm{f_{\phi_{t+w}}(y') - f_{\phi_{t}}(y')} +  \norm{\lblencPTstep{t+w} - f_{\phi_{t+w}}(y')} - \norm{\lblencPTstep{t+w} - f_{\phi_{t+w}}(y')} \right | \\
    &= \left | \norm{\lblencPTstep{t} - \lblencPTstep{t+w}} +  \norm{f_{\phi_{t}}(y') - f_{\phi_{t+w}}(y')} \right |  
\end{align}
Recall that the dual encoders satifies the Lipschitz assumption (Assumption~\ref{assumption:lipchitz}):
\begin{align}
      \norm{\lblencPTstep{t} - \lblencPTstep{t+w}}  &\leq  L\norm{\Theta-\Theta'} \leq w \eta 2 \invtemp L M \ \ \ \forall y
\end{align}

And so:
\begin{align}
& \left | \norm{\lblencPTstep{t} - f_{\phi_{t}}(y')} + \norm{\lblencPTstep{t+w} - f_{\phi_{t+w}}(y')} \right | \leq 4w\eta \invtemp LM. 
\end{align}
\end{proof}

\section{EFFICIENT RE-ENCODING OF TARGETS}
\label{appendix:LowDim}

In this section, we describe in more detail
the re-encoding part of \ours. 
Let's recap the structure of the re-encoding model.
Recall that we keep a cached, low-dimensional version of each encoded target. We denote this cache as  $Y_t'\subset \mathbb{R}^{d'}$ where $d'$ is the number of landmarks used in \nystrom\ / the reduced dimensionality and $t$ represents the number of gradient steps. Note that while we denote separate caches for each time step, we need not physically store multiple such caches. The cache from the previous time step can be overwritten during the update. 

We refer to $Y_t$ as the set of target embeddings in their full dimensional space. This set $Y_t$ is never fully instantiated. 
We have a set of $d'$ \emph{landmark} points, which are kept fixed during training and which were sampled uniformly at random. We will store these landmarks in their full dimensional form. We denote the set of landmarks as $\mathscr{S}$. We also have a set of training points used to fit the regression model, $\mathscr{R}$, which maps target embeddings in $Y_t$ to form $Y_{t+w}$. Each time we fit $\mathscr{R}$, we sample $s'$ targets uniformly at random and build a training dataset for the regressor using the updated dual encoder, $f_{\phi_{t+w}}$, specifically, $(\vec{y_1},\lblencPTstepNoLbl{t+w}(y_1)),\dots(\vec{y_{s'}},\lblencPTstepNoLbl{t+w}(y_{s'}))$. We only need to store the regression training dataset points and the landmarks in their full dimension form. For all other targets, we can compute their approximate full dimensional representation using $\mathscr{R}$ and then immediately project into lower dimensional space $d'$ using the landmarks/\nystrom. We define the regression model $\mathscr{R}$ to be a low-rank, \nystrom, regression model with the same landmarks $\mathscr{S}$. Therefore, the input to the regression model can be the low-dimensional representations. 

In summary, the approximate re-encoding procedure would,  re-encode $s'$ points using the new encoder model, $\lblencPTstepNoLbl{t+w}$, build a training dataset for the regression model $\mathscr{R}$, fit the regression model $\mathscr{R}$, use $\mathscr{R}$ to update the landmark point representation $\mathscr{S}$, use  $\mathscr{R}$ and $\mathscr{S}$ to build the low-dimensional embeddings for all targets in $Y_t'$ as well as the corresponding $(\mathbf{S}^T\mathbf{K}\mathbf{S})^{-1}$ projection matrix. This is summarized in Algorithm~\ref{alg:reEncode} (with initialization in Algorithm~\ref{alg:initialEncoding}), which is called as a subroutine of the overall training algorithm (Algorithm~\ref{alg:dynnibalComplete}).

It is important to note that we do not use these approximate
re-encoded targets at evaluation time. We only use these 
dimensionality reduced targets at training time. 

Lastly, we note that while we sample the landmark points from only the targets empirically, one could (in a more principled way) sample from all datapoints and targets.

\subsection{Building SG Trees with \nystrom Representations}
\label{sec:nystromSG}

When we use our \nystrom-based lower dimensional representations, the unit norm assumptions about the targets no longer apply. At query time, for a given point $x$, we can think about it as corresponding to some row $i$ in the $\mathbf{KS}$ matrix. Similarly each target $y$ to some column $j$ in $(\mathbf{S}^T\mathbf{K}\mathbf{S})^{-1}\mathbf{S}^T\mathbf{K}$. This corresponds to using the $(\mathbf{S}^T\mathbf{K}\mathbf{S})^{-1}\mathbf{S}^T\mathbf{K}$ representations as the target cluster representatives in the tree structure. However, when building (or re-building) the tree structure, we need to be able to measure the similarity between two targets, i.e., using both $\mathbf{KS}$ and $(\mathbf{S}^T\mathbf{K}\mathbf{S})^{-1}\mathbf{S}^T\mathbf{K}$. We only need to use the $\mathbf{KS}$ representations during construction / re-building time however and if memory is a concern, we could only store a single representation and use the $(\mathbf{S}^T\mathbf{K}\mathbf{S})^{-1}$ when measuring (dis)similarity. 

Finally, we need to be able to use the SG Tree despite using inner product similarities (rather than a proper distance measure). Rather than the approach presented by \cite{zaheer2019sg}, we find that converting to distance by $\exp(-\mathbf{KS}(\mathbf{S}^T\mathbf{K}\mathbf{S})^{-1}\mathbf{S}^T\mathbf{K})$ is more effective. We find that setting the base of the SG tree to be 1.05 seems to work well with this conversion of similarities to distances. Note that these modifications mean that even running the exact nearest neighbor search algorithm may result in approximations. However, we find empirically that this structure is sufficient for training dual encoders.

\begin{algorithm}
\caption{Approximate Re-Encode}
\begin{algorithmic}[1]
\STATE Let $\mathscr{S}$ refer to the set of landmark points 
\STATE Sample $s'$ targets uniformly at random from the collection of targets to form $S'_t$.
\STATE Let $Y'_{S'_t}$ be the low dimensional representations of the training points. 
\STATE Build regression training dataset  $(\vec{y_1},\lblencPTstepNoLbl{t+w}(y_1)),\dots(\vec{y_{s'}},\lblencPTstepNoLbl{t+w}(y_{s'}))$ using the updated encoder model. 
\STATE Train $\mathscr{R}$ on $(\vec{y_1},\lblencPTstepNoLbl{t+w}(y_1)),\dots(\vec{y_{s'}},\lblencPTstepNoLbl{t+w}(y_{s'}))$ \mycomment{$\mathscr{R}$ is a low-rank regression model using the same landmarks $\mathscr{S}$.}
\STATE Update the embedding of landmark representatives $\mathscr{S}$ using $\mathscr{R}$. Update \nystrom projection matrix $(\mathbf{S}^T\mathbf{K}\mathbf{S})^{-1}$
\STATE Produce new low-dimensional embeddings for all targets $Y'_{t+w}$ using $\mathscr{R}$.
\end{algorithmic}
\label{alg:reEncode}
\end{algorithm}

\begin{algorithm}
\caption{Initial Target Encoding}
\begin{algorithmic}[1]
\STATE Sample landmarks $\mathscr{S}$ uniformly at random
\STATE Encode landmarks, fit \nystrom.
\STATE Encode all of the targets with the encoder model $\lblencPTstepNoLbl{}$ and use \nystrom to produce low dimensional representations $Y'_{0}$
\end{algorithmic}
\label{alg:initialEncoding}
\end{algorithm} 
\section{ADDITIONAL EMPIRICAL DETAILS}
\label{app:empirical_details}

The dual encoder models used in all experiments are transformers, initialized with the Roberta base model \citep{liu2019roberta}. We use the hyperparameters presented in Table~\ref{tab:hyperparameters}. We note that the changes unique for Zeshel were done to account for the longer data point input length.

\begin{table}[]
    \centering
    \begin{tabular}{l|l}
    \toprule
    \bf Name & \bf Value  \\
    \midrule
    Train Batch Size & 128 total\\
    Uniform Negatives & 64 per training example (32 for Zeshel)  \\
    Sampled Negatives & 64 per training example \\
    Initial Warmup Learning Rate & 1e-5 \\
    End Warmup Learning Rate & 1e-7 \\
    Warmup Steps & 10,000 \\
    Warmup Decay Type & Polynomial \\
    Optimizer & Adam \\
    Optimizer Beta1 & 0.9 \\
    Optimizer Beta2 & 0.999 \\
    Optimizer Epsilon & 1e-8 \\
    Score Scaling & 20.0 \\
    Max Data Point Feature Length & 128 (256 for Zeshel) \\
    Max Target Feature Length & 256 \\
    Encoder Embedding Dimension & 768 \\
    \bottomrule
    \end{tabular}
    \caption{Hyperparameters used for training dual encoder models.}
    \label{tab:hyperparameters}
\end{table}

All models unit-norm the output embedding that is produced by the dual encoder. All models perform score-scaling as in \cite{lindgren2021efficient} with a scaling factor of 20.0. All models share sampled negative examples across all positives in the same batch. In-batch negatives are used by all models. Stochastic Negative mining and \ours use uniform negatives in addition to the sampled hard negatives.

We first train models using uniform negatives and then apply \ours. Note that this mimics the hypothesis that the softmax distribution would be closer to uniform in the beginning of training. We perform 4000 steps of uniform negatives for NQ and 8000 MSMARCO. 

We use 128 landmarks in \nystrom/regression model for all datasets. For Natural Questions and MSMARCO we use 8192 training examples for the regression model. For Zeshel we use 256 training examples for the regression model. For NQ, we perform a full refresh of all embeddings (re-initializing \nystrom/regression model) once during training at step 7500.

We find that changing the number of landmarks used does not dramatically change the performance of the method. For instance, on Natural Questions (NQ), we find that in terms of R@1, 256 landmarks achieves 0.488
while with 128 landmarks gets 0.481 and 64 landmarks gets 0.479 and R@100 with 256 landmarks is 0.862 and 128
landmarks gets 0.859 and 64 landmarks gets 0.859.

\subsection{SG Tree Details}

We use the \nystrom-modified SG Tree that is described in \S\ref{sec:nystromSG}. We notice that the tree construction and sampling procedure can be done much more efficiently on smaller trees. And so, rather than constructing one tree over all the targets, we construct a forest of 100 trees of roughly equal size. We construct and build samples from this forest independently and aggregate samples (by taking the max unnormalized probability).

\section{ZESHEL EXPERIMENTAL RESULTS}
\label{app:zeshel}

\noindent \textbf{Zeshel} \citep{logeswaran2019zero} is a dataset for classifying (linking) ambiguous mentions of entities to their unambiguous entity pages in Fandom Wikias. The original dataset separates individual Fandom Wikia into separate domains, which severely limits the number of targets. To make the task more challenging and better suited for evaluation of models approximating the softmax loss, we combine all Wikias together, resulting in a collection of 492K targets. We evaluate in terms of recall of the correct entity.

Tables \ref{tab:zeshel_dev_table} and \ref{tab:zeshel_test_table} report the dev and test scores for the Zeshel dataset. Recall that the number of targets in this dataset is significantly less than the other two datasets ($\sim$500K compared to 21M (Natural Questions), 8.8M (MSMARCO)). We find that the performance of \ours improves upon baselines of Stochastic Negative Sampling and Uniform Negatives. 

Because of the reduced size of the dataset, we use only 256 training points for the regression model. We also use a relatively small number of examples for stochastic negative mining (32768 in total), though this accounts for a larger memory percentage overall because of the reduced number of targets.

\begin{table}[]
    \centering
    \begin{tabular}{@{}l c c c c @{}}
    \toprule
        Recall & Mem. \%  & R@1 & R@10 & R@100 \\
         \midrule
         In-batch Negatives & 0\% &  0.388 & 0.726 & 0.861 \\
         Uniform Negatives& 0\% &  0.393 & 0.713 & 0.851 \\
         Stochastic Neg & 6.59\% & 0.413 & 0.737 & 0.865 \\
         \ours & 0.064\% &  0.421  & 0.742 & 0.870  \\
         \midrule
         Exhaustive & 100\% &  0.444 & 0.756  & 0.880 \\
    \bottomrule
    \end{tabular}
    \caption{\textbf{Zeshel (Dev) Results}. We measure the recall performance on the entity linking dataset. We note that this dataset is considerably smaller than the other datasets in terms of number of targets. Still, we find that \ours can achieve better performance than baseline methods while approaching the performance of the brute force oracle. \ours achieves better performance that a Stochastic Negative Sampling baseline which uses more memory.}
    \label{tab:zeshel_dev_table}
\end{table}
\begin{table}[]
    \centering
    \begin{tabular}{@{}l c c c c @{}}
    \toprule
        Recall & Mem. \%  & R@1 & R@10 & R@100 \\
         \midrule
         In-batch Negatives & 0\% & 0.344 &0.627   &  0.778  \\
         Uniform Negatives& 0\% & 0.323  & 0.593  & 0.747  \\
         Stochastic Neg & 6.59\% & 0.330 & 0.609 & 0.756  \\
         \ours & 0.064\% &  0.348  & 0.623 & 0.766  \\
         \midrule
         Exhaustive & 100\% &  0.348 & 0.617  & 0.767 \\
    \bottomrule
    \end{tabular}
    \caption{\textbf{Zeshel (Test) Results}. We measure the recall performance on the entity linking dataset. While the dataset has considerably fewer targets, \ours achieves nearly the same result as the exhaustive brute force method and better results than Stochastic Negative Mining and Uniform Negatives. In this setting, In-Batch Negatives performs very well, perhaps because of the relatively small number of targets.}
    \label{tab:zeshel_test_table}
\end{table}

 \section{LIMITATIONS}

The focus of this work is on improving the approximation of the cross-entropy loss and the quality of the samples from the softmax distribution. For some tasks, it may be the case that different objectives, including multi-task, pretraining, contrastive, and others may lead to models that generalize better. After applying \nystrom, the approximate re-encoding step, and the sampling approximations, we lose theoretical properties. Future work could investigate both why these approximations work well and how to bound their approximation quality.

\section{ETHICAL CONSIDERATIONS}

The proposed approach is subject to the same biases and ethical considerations as the dual encoders used as the base model. The reduction in error cannot be assumed to reduce or exacerbate the potentially negative aspects of such an encoder model. The biases and considerations of any such retrieval/classification task need to be carefully considered as with any model. The implications of negative sampling via uniform vs. hard methods will likely impact the decision boundary of the model (as observed in our empirical experiments). Understanding how the training method relates to biases in the data, labels, targets, would be an important consideration.

\begin{table}[]
    \centering
    \begin{tabular}{@{}ll@{}}
    \toprule
    \bf Symbol & \bf Definition \\
    \midrule
        $\datapoint$ & A data point / the features of a data point, e.g. input to the data dual encoder. \\
        $\dataset$ & The set of data points \\
        $\lbl$ & A target / the features of a target, e.g. input to the target dual encoder. \\
        $\labelset$ & Set of all targets \\
        $\loss$ & Loss function, in particular cross entropy \\
        $\hat{\loss}$ & Our approximate, sampling based loss function \\
        L & Lipschitz value, Assumption~\ref{assumption:lipchitz} \\
        M & Bound on gradient of logits, Assumption~\ref{assumption:boundedGradient}  \\
        $\dataenc$ & The data point dual encoder \\
        $\lblenc$ & The target dual encoder \\
         $\Theta$ & The parameters of both the data and target dual encoders. \\
         \midrule
        $Z$ & The partition function for the softmax.    Eq~\ref{eq:softmax}. \\
        $\hat{Z}$ & The approximated partition function using the clustering-based approach. Eq~\ref{eq:approxclustersoftmax}. \\
        $Z_k$ & The normalizer constant when descending to level $k$ in the rejection sampling approach. \\
        $P$, $P(y|x)$ & True, exact softmax distribution \\
        $\approxsoftmax$ & The proposal distribution for Metropolis-Hastings. Eq~\ref{eq:approxclustersoftmax}. \\
        $\approxsoftmax_\textsf{MH}$ & The distribution over labels given by Metropolis Hastings sampling procedure. See Appendix~\ref{sec:mh_appendix}.\\
        $\expNegRejError$ & Used for unnormalized probability error term in rejection sampling \\
        $d$ & Dimensionality of dual encoder output. \\
         \midrule
        $\partition$ & A clustering of the targets. \\
        $\lbl^{(\partition)}$ & The cluster assignment of target $\lbl$. Overloaded to also be the cluster's representative (or its features).  \\
        $\cluster$ & A cluster of targets, e.g. $\cluster \subseteq \labelset$. \\
        $\lblenc(\lbl^{(\partition)})$,$\lblenc(\cluster)$ & The encoded representation of the cluster's representative. \\
        $\hc$ & A hierarchical clustering / cover or SG Tree of the targets \\
        $\ctbase$ & The base parameter of the cover / SG tree \\
        $\children{(\cluster)}$ & The children of a node in the hierarchical clustering/cover tree. \\
        $\partition_{(\ell)}$ & The partition associated with level $\ell$ of the cover tree. \\
        $Y_{(\ell)}$ & The set of cluster representatives for the partition associated with level $\ell$ of the cover tree. \\
        $\alpha$ & The expansion constant used in our theoretical analysis~\ref{def:expansionConstant}\\
        $\textsc{maxd}(\cluster)$ & The maximum distance between the cluster representative of $\cluster$ and one of its descendants in an SG Tree. \\ 
        $\textsc{mind}(\cluster)$ & The minimum distance between one of the pairs of children of $\cluster$ in an SG Tree. \\ 
        \midrule
        $\gamma$ & The upper bound on the ratio of the true softmax distribution, $\max_y \frac{P(y|x)}{Q(y|x;\partition)} \leq \gamma$ \\
        $w$ & Number of steps of gradient descent. \\
        $s$ & The Metropolis Hastings chain length \\
        $\eta$ & Gradient descent learning rate \\
        $d_{max}$ & Maximum pairwise distance among all pairs of points \\
        $d_{min}$ & Minimum pairwise distance among all pairs of points \\
        \midrule
        $\mathbf{S}$ & Binary matrix representing landmarks for \nystrom \\
        $\mathscr{S}$ & Set of landmarks for \nystrom \\
        $\mathbf{K}$ & Pairwise similarity matrix for \nystrom \\
        $\mathscr{R}$ & Low-rank (\nystrom) regression model to approximately re-encode the targets \\
        $Y'_{S'_t}$ & The low dimensional representations of the training points for $\mathscr{R}$. Not to be confused \\
        &  with $Y_{(\ell)}$, the representatives in a level of the cover tree.  \\
         \bottomrule
    \end{tabular}
    \caption{\textbf{Summary of Notation Used}}
    \label{tab:notation_table}
\end{table}

\end{document}